%% file: main.tex
\documentclass[lettersize,journal]{IEEEtran}
\usepackage{amsmath,amsfonts}
\usepackage{algorithmic}
\usepackage{algorithm}
\usepackage{array}
\usepackage[caption=false,font=normalsize,labelfont=sf,textfont=sf]{subfig}
\usepackage{textcomp}
\usepackage{stfloats}
\usepackage{url}
\usepackage{verbatim}
\usepackage{graphicx}
\usepackage{cite}
\usepackage{amsmath}
\usepackage{makecell}
\usepackage{booktabs}
\allowdisplaybreaks[4]
\usepackage{ntheorem}
\newtheorem{theorem}{\bf Theorem}
\newtheorem{remark}{\bf Remark}

\newtheorem{definition}{\bf Definition}
\newtheorem{assumption}{\bf Assumption}
\newtheorem{lemma}{\bf Lemma}
\newtheorem*{proof}{Proof.}
\usepackage{tabularx}
\usepackage{threeparttable}
\usepackage{float}
\usepackage{multirow}
\usepackage{amssymb}
\usepackage{bbding}
\usepackage{utfsym}
\usepackage{array}
\usepackage{wrapfig}

\begin{document}

\title{Towards Understanding Generalization and Stability Gaps between Centralized and Decentralized Federated Learning}

\author{Yan Sun, Li Shen, and Dacheng Tao,~\IEEEmembership{Fellow,~IEEE}
\thanks{
Yan Sun is with the University of Sydney, Australia.\(Email: woodenchild95@outlook.com\). }
\thanks{
Li Shen is with Sun Yat-sen University, China. \(Email: mathshenli@gmail.com\).}
\thanks{ 
Dacheng Tao is with Nanyang Technological University, Singapore. \(Email: dacheng.tao@ntu.edu.sg\).}       
\thanks{Preprint.}}



\maketitle

\input{texts/abstract}
\input{texts/introduction}
\input{texts/related_work}
\input{texts/methodology}
\input{texts/theoretical_analysis}
\input{texts/experiments}
\input{texts/conclusion}


\bibliographystyle{IEEEtran} 
\bibliography{main}

\section{Biography Section}
\vspace{11pt}

\begin{IEEEbiography}[{\includegraphics[width=1in,height=1.25in,clip,keepaspectratio]{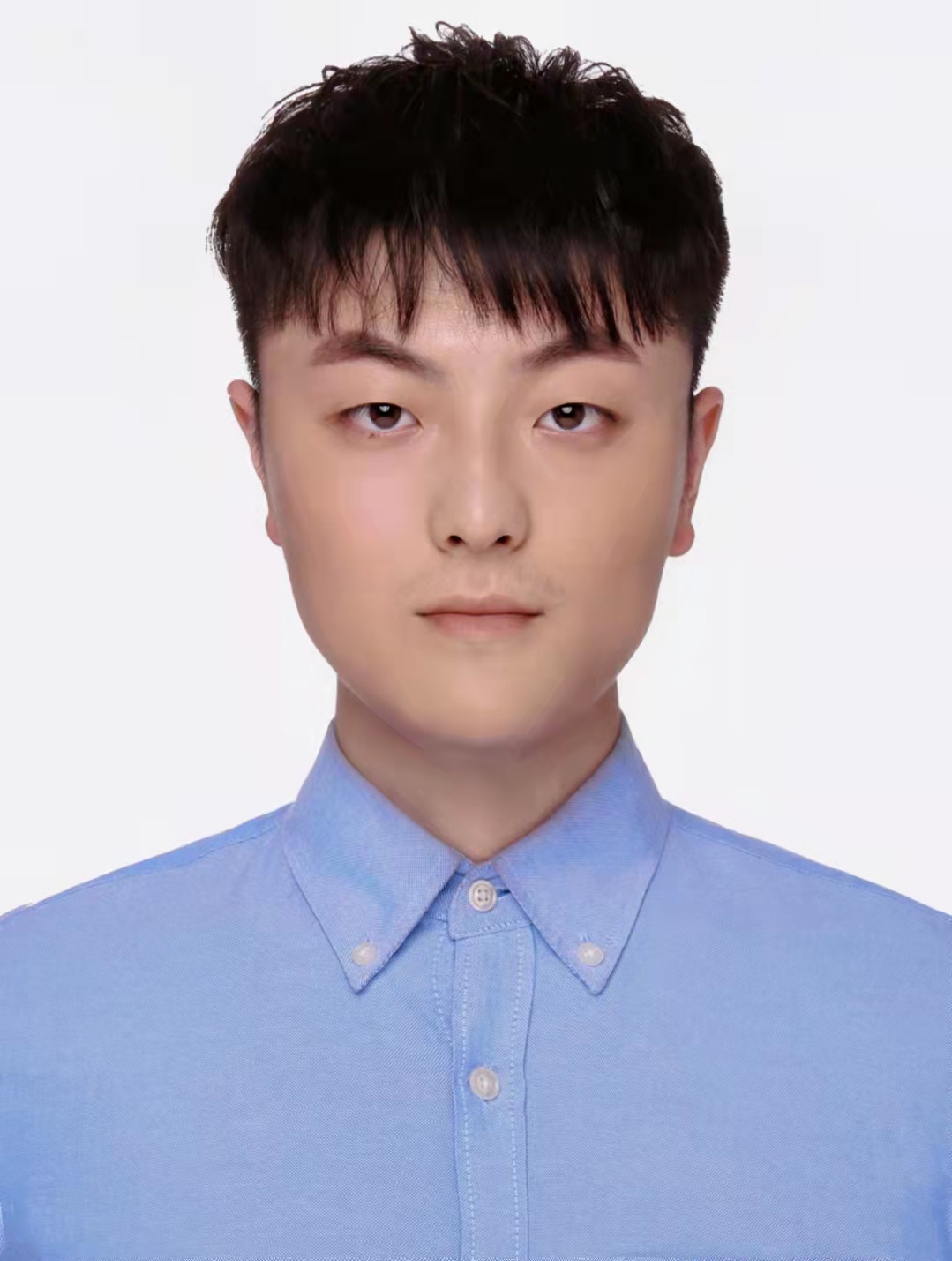}}]{Yan~Sun} received his B.Eng. degree in electronic and information engineering from University of Science and Technology of China (USTC) in 2018. He is currently a Ph.D. candidate at the University of Sydney (USYD). His current research area includes large-scale training, federated learning, machine learning, and learning theory.
\end{IEEEbiography}

\begin{IEEEbiography}
[{\includegraphics[height=1.25in,clip,keepaspectratio]{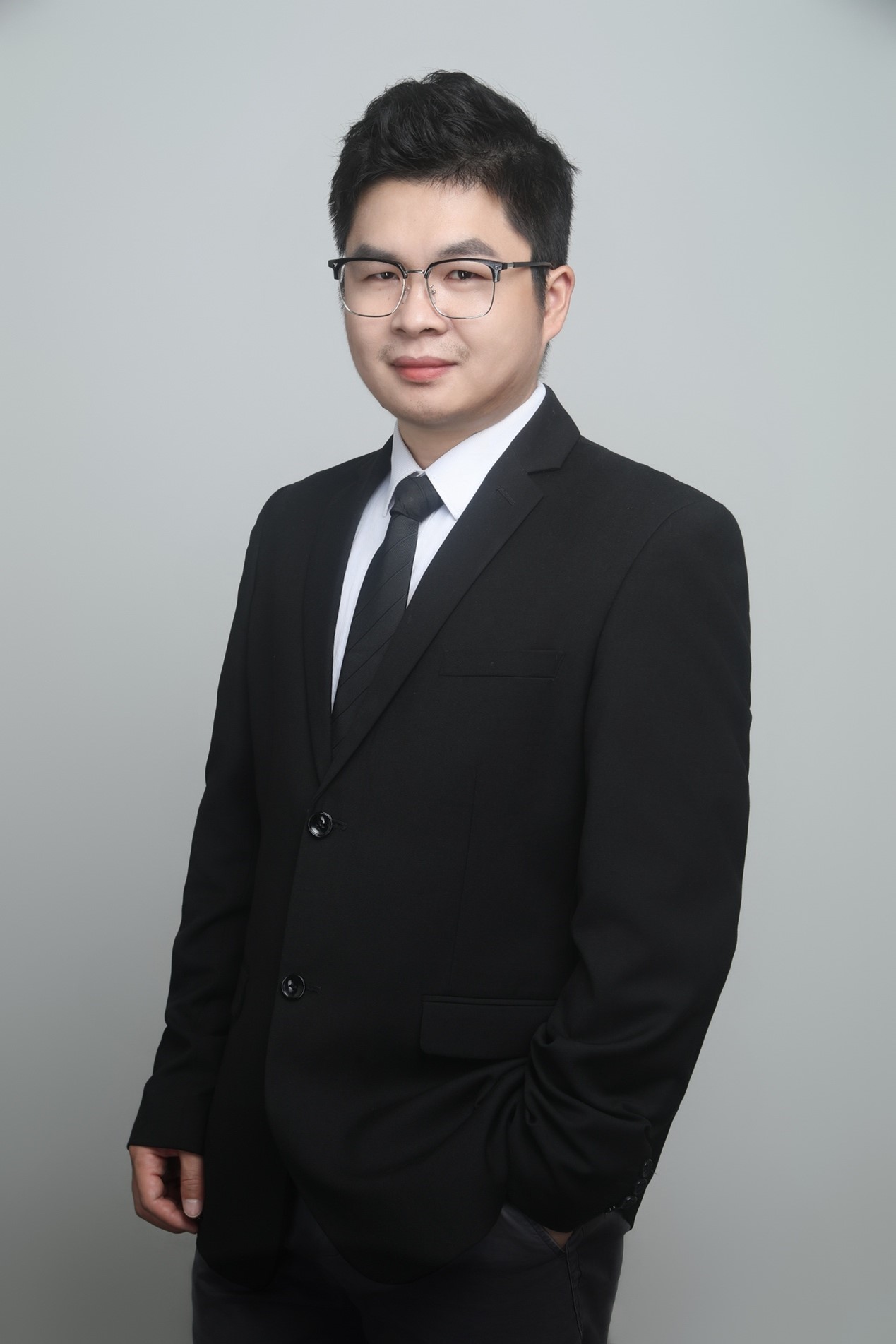}}]
{Li~Shen} is currently an associate professor at Sun Yat-sen University. Previously, he was a research scientist at JD Explore Academy, Beijing, and a senior researcher at Tencent AI Lab, Shenzhen. He received his bachelor's degree and Ph.D. from the School of Mathematics, South China University of Technology. His research interests include theory and algorithms for nonsmooth convex and nonconvex optimization, and their applications in trustworthy artificial intelligence, deep learning, and reinforcement learning. He has published more than 100 papers in peer-reviewed top-tier journal papers (JMLR, IEEE TPAMI, IJCV, IEEE TSP, IEEE TIP, IEEE TKDE, etc.) and conference papers (ICML, NeurIPS, ICLR, CVPR, ICCV, etc.). He has also served as the senior program committee for AAAI and area chairs for ICML, ICLR, ACML, and ICPR.
\end{IEEEbiography}

\begin{IEEEbiography}
[{\includegraphics[height=1.25in,clip,keepaspectratio]{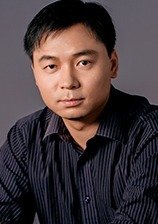}}]
{Dacheng~Tao} is currently a Distinguished University Professor in the College of Computing \& Data Science at Nanyang Technological University. He mainly applies statistics and mathematics to artificial intelligence and data science, and his research is detailed in one monograph and over 200 publications in prestigious journals and proceedings at leading conferences, with best paper awards, best student paper awards, and test-of-time awards. His publications have been cited over 112K times and he has an h-index 160+ in Google Scholar. He received the 2015 and 2020 Australian Eureka Prize, the 2018 IEEE ICDM Research Contributions Award, and the 2021 IEEE Computer Society McCluskey Technical Achievement Award. He is a Fellow of the Australian Academy of Science, AAAS, ACM and IEEE. 
\end{IEEEbiography}
\vspace{11pt}
\vfill

\newpage
\input{texts/appendix}

\end{document}

%% file: texts/abstract.tex
\begin{abstract}
\label{abstract}
As two mainstream frameworks in federated learning~(FL), both centralized and decentralized approaches have shown great application value in practical scenarios.
However, existing studies do not provide sufficient evidence and clear guidance for analysis of which performs better in the FL community.
Although decentralized methods have been proven to approach the comparable convergence of centralized with less communication, their test performance always falls short of expectations in empirical studies.
To comprehensively and fairly compare their efficiency gaps in FL, in this paper, we explore their stability and generalization efficiency. 
Specifically, we prove that on the general smooth non-convex objectives, 1) centralized FL~(CFL) always generalizes better than decentralized FL~(DFL); 2) CFL achieves the best performance via adopting partial participation instead of full participation; and, 3) there is a necessary requirement for the topology in DFL to avoid performance collapse as the training scale increases. 
We also conduct extensive experiments on several common setups in FL to validate that our theoretical analysis is consistent with experimental phenomena and contextually valid in several general and practical scenarios.
\end{abstract}

\begin{IEEEkeywords}
Generalization, stability, centralized and decentralized, federated learning.
\end{IEEEkeywords}

%% file: texts/introduction.tex
\section{Introduction}
\label{introduction}

\IEEEPARstart{S}{ince} \cite{mcmahan2017communication} propose FL, it becomes a promising paradigm for training the heterogeneous dataset. Classical {\ttfamily FedAvg} utilizes the FL paradigm consisting of a global server and massive local clients, to jointly train a global model via periodic communication. Though it shines in large-scale training, CFL has to afford expensive communication costs under a large number of local clients. To efficiently alleviate this pressure, DFL is introduced as a compromise. \cite{sun2022decentralized} learn the {\ttfamily D-FedAvg} method and analyze its fundamental properties, which adopts a communication topology across all clients to significantly reduce the number of links. Sometimes it is impossible to set up the global server, in which case DFL becomes the only valid solution. As two major frameworks in the current FL community, both centralized and decentralized approaches are well studied and greatly improved. More and more insights are being revealed to give it huge potential for applications. However, there is still a \textit{question} lingering in the general FL studies which hinders further developments:
\begin{center}
    \textit{How to quantitatively compare their performance in the context of federated learning?}
\end{center}

\IEEEpubidadjcol

Research on this question goes back to distributed learning based on the parallel stochastic gradient descent~(PSGD). \cite{lian2017can} study the comparison between centralized PSGD~(C-PSGD) and decentralized PSGD~(D-PSGD), and provide a positive answer for decentralized approaches. D-PSGD can achieve a comparable convergence rate with linear speedup as the C-PSGD with much fewer communication links. However, due to the heterogeneous dataset and local training process, DFL always performs consistently poor experimental performance. \cite{shi2023improving} also learn that one algorithm often performs worse on the DFL framework under the same experimental setups. Although some studies provide consequential interpretations based on consensus analysis, there is still a lack of solid analysis to explain these experiments, especially for generalization.

Most of the previous works focus on the analysis of convergence rates and ignore the generalization efficiency, while the test accuracy is highly related to the generalization. Therefore, the incomplete comparison will easily lead to cognitive misunderstandings. To further understand their performance differences and comprehensively answer the above \textit{question}, we utilize the analysis of uniform stability~\cite{elisseeff2005stability,hardt2016train} to learn their generalization gaps. To improve the applicability of the vanilla uniform stability analysis in general deep models, we remove the idealized assumption of bounded full gradients and adopt bounded uniform stability instead.

\begin{table*}[t]
  \caption{Main results on the generalization of {\ttfamily FedAvg} and {\ttfamily D-FedAvg} on the smooth non-convex objective. $\kappa_\lambda$ is a constant related to spectrum gap $1-\lambda$~(Spectrum norm $\lambda$ is defined in Definition~\ref{def_adjacent_matrix}). We mainly focus on the impacts of $K$, $T$, $S$, and $m$ terms.}
  \label{Comparison_FL_DFL}
  \centering
  \resizebox{1\textwidth}{!}{
  \begin{tabular}{c|c|ccc|ccc}
    \toprule
         & Generalization error $\varepsilon_G$  &  Optimal  & $\varepsilon_G$ & Equivalent & Worst  & {$\varepsilon_G$} & Equivalent \\
    \midrule
    {\ttfamily FedAvg} & $\mathcal{O}\left(\frac{1}{S}\left(\frac{n^{\frac{\mu L}{1 + \mu L}}}{m}\right)\left(KT\right)^{\frac{\mu L}{1 + \mu L}}\right)$ & $n=1$ & $\mathcal{O}\left(\frac{\left(KT\right)^{\frac{\mu L}{1 + \mu L}}}{Sm}\right)$ & SGD & $n=m$ & $\mathcal{O}\left(\frac{\left(KT\right)^{\frac{\mu L}{1 + \mu L}}}{Sm^{\frac{1}{1 + \mu L}}}\right)$ & \makecell{DFL \\ (full)}\\
    \specialrule{0em}{2pt}{2pt}
    \cline{1-8}
    \specialrule{0em}{2pt}{2pt}
    {\ttfamily D-FedAvg} & $\mathcal{O}\left(\frac{1}{S}\left(\frac{1+6\sqrt{m}\kappa_\lambda}{m}\right)^{\frac{1}{1 + \mu L}}\left(KT\right)^{\frac{\mu L}{1 + \mu L}}\right)$ & $\kappa_\lambda = 0$ & $\mathcal{O}\left(\frac{\left(KT\right)^{\frac{\mu L}{1 + \mu L}}}{Sm^{\frac{1}{1 + \mu L}}}\right)$ & \makecell{CFL \\ (full)} & - & - & - \\
    \bottomrule
  \end{tabular}
  }
  \begin{tablenotes}
  \small
  \centering
    \item $n$: the number of clients participating in the training per round in CFL; $m$: the number of total clients;
    \item $K$: local interval; $T$: communication rounds; $L$: Lipschitz constant; $\mu$: a general constant with $\mu\leq 1/L$;
  \end{tablenotes}
\end{table*}
Our work provides a novel and comprehensive understanding of generalization comparisons between CFL and DFL as shown in Table~\ref{Comparison_FL_DFL}. \cite{hardt2016train} learn the general stability of SGD~(M) methods in the non-convex optimization increases further as the number of training samples increases. When we consider the federated learning scenario, we also need to further account for the negative influence of its inherent properties, i.e. local training process and aggregation on large-scale private clients. In this paper, we isolate the impact of these components and discuss the corresponding conclusions through the analysis of the worst-case stability in each scenario. These conclusions reveal a novel understanding in the field of federated learning. The main extra factors affecting the stability of C/DFL comes from data distributions, i.e. the number of clients and dataset size of each client, and the communication links, i.e. connection topology. We also provide a detailed analysis and corresponding inferences on both frameworks.

As shown in Table~\ref{Comparison_FL_DFL}, due to periodic local training processes, stability of both {\ttfamily FedAvg} and {\ttfamily D-FedAvg} methods are slower than general SGD~(M) method. On the efficiency of the dataset size, under total $Sm$ samples, {\ttfamily FedAvg} achieves $\mathcal{O}\left(\frac{n^{\frac{\mu L}{1+\mu L}}}{Sm}\right)$ rate and {\ttfamily D-FedAvg} achieves $\mathcal{O}\left(\frac{1}{S}\left(\frac{1+6\sqrt{m}\kappa_\lambda}{m}\right)^\frac{1}{1+\mu L}\right)$ rate, respectively. In the CFL framework, its negative impact is mainly caused by the partial participation ratio $\frac{n}{m}$. When the number of participating clients is decreased to $1$, the {\ttfamily FedAvg} method completely degenerates into SGD~(M). In the DFL framework, the negative impact is mainly caused by the topology of communication, which corresponds to the different coefficients $\kappa_\lambda$. When applied to a fully connected topology, {\ttfamily D-FedAvg} degenerates into {\ttfamily FedAvg} under full participation training mode~($n=m$). From the generalization perspective, CFL always generalizes no worse than DFL. The best generalization error in DFL is training with the {\ttfamily full}-topology~($\kappa_\lambda = 0$), which is equivalent to {\ttfamily full}-participation in CFL~($n = m$). However, {\ttfamily full}-participation in CFL leads to the worst generalization error. Moreover, our analysis not only quantifies the differences between them but also demonstrates some discussions on the negative impact of topology in DFL. As a compromise of CFL to save communication costs, the topology adopted in DFL has a specific minimum requirement~($\kappa_\lambda\leq\mathcal{O}\left(\sqrt{m}\right)$) to avoid performance collapse when the number of local clients $m$ increases. We also conduct extensive experiments on the widely used FL setups to validate our analysis. Both theoretical and empirical studies confirm the validity of our answers to the \textit{question} above.
We summarize the main contributions of this paper as:
\begin{itemize}
    \item We provide the uniform stability analysis for the {\ttfamily FedAvg} and {\ttfamily D-FedAvg} algorithms without adopting the idealized assumption of bounded full gradients.
    \item We prove centralized approaches always generalize no worse than decentralized ones, and CFL only needs partial participation to achieve optimal test error. 
    \item We prove even with adopting DFL as a compromise of CFL, there is a minimum requirement on the topology. Otherwise, even with more local clients and data samples participating in the training process, its generalization performance still gets worse.
    \item We conduct extensive experiments to validate our theoretical analysis proposed in this paper.
\end{itemize}

The subsequent content of this paper is organized as follows. Sec.~\ref{related work} introduces related works. Sec.~\ref{methodology} introduces the preliminaries and problems studied in this paper. Sec.~\ref{theoretical analysis} demonstrates the main theorems and Sec.~\ref{experiments} shows the empirical studies and experimental supports. The additional experiments and relevant proofs are stated in the appendix.

%% file: texts/related_work.tex
\section{Related Work}
\label{related work}

\textbf{Centralized federated learning.} Since \cite{mcmahan2017communication} propose the fundamental CFL algorithm {\ttfamily FedAvg}, several studies explore its strengths and weaknesses.
\cite{yang2021achieving} prove its convergence rate on the non-convex and smooth objectives satisfies the linear speedup property. 
Furthermore, \cite{karimireddy2020scaffold} study the client-drift problems in FL and adopt the variance reduction technique to alleviate the local overfitting.
\cite{li2020federated} introduce the proxy term to force the local models to be close to the global model.
\cite{zhang2021fedpd,acar2021federated,gong2022fedadmm,sun2023fedspeed} study the primal-dual methods in CFL and prove it achieves faster convergence. This is also the optimal convergence rate that can be achieved in the current algorithms.
With the deepening of research, researchers begin to pay attention to its generalization ability. 
One of the most common analyses is the PAC-Bayesian bound. 
\cite{yuan2021we} learn the components in the generalization and formulate them as two expectations.
\cite{reisizadeh2020robust} define the margin-based generalization error of the PAC-Bayesian bound.
\cite{qu2022generalized,caldarola2022improving,sun2023fedspeed} study the generalization efficiency in CFL via local sharpness aware minimization.
\cite{sun2023dynamic} re-define the global margin-based generalization error and discuss the differences between local and global margins.
\cite{sefidgaran2023federated} also learn the reduction of the communication may improve the generalization performance.
In addition, uniform stability~\cite{elisseeff2005stability,hardt2016train} is another powerful tool adopted to measure the generality.
\cite{yagli2020information} learn the generalization error and privacy leakage in federated learning via the information-theoretic bounds.
\cite{sun2023understanding} provide the stability analysis for several FL methods.
\cite{fedinit} prove stability in CFL is mainly affected by the consistency.

\quad 

\textbf{Decentralized federated learning.} Before the exploration of the decentralized federated learning community, decentralized distributed implementations have developed for a long period. \cite{JMLR:v22:20-147} learn a unified framework on the local-updates-based methods, and then \cite{yuan2020influence} explore the impact of the bias correction in a distributed framework and reveal its efficiency in the training. Based on \cite{shi2015extra} which analyzes the consensus in decentralized approaches, \cite{alghunaim2022unified} provide a novel unified analysis for the non-convex decentralized learning. After that, local updates also draw much attention to efficient training. \cite{mishchenko2022proxskip,nguyen2022performance,alghunaim2023local} learn the advantages of the local process in both centralized and decentralized approaches.
With deep studies of the full utility of edged devices, decentralized federated learning becomes a promising application.
Since \cite{lalitha2018fully} propose the prototype of DFL, it is becoming a promising approach as the compromise of CFL to save the communication costs.
\cite{lian2018asynchronous,yu2019linear,assran2019stochastic,koloskova2020unified} learn the stability of decentralized SGD which contributes to the research on the heterogeneous dataset.
\cite{hu2019decentralized} study the gossip communication and validate its validity. 
\cite{hegedHus2021decentralized} explore the empirical comparison between the prototype of DFL and CFL.
\cite{lim2021decentralized} propose a dynamic resource allocation for efficient hierarchical federated learning.
\cite{sun2022decentralized} propose the algorithm {\ttfamily D-FedAvg} and prove that it achieves the comparable convergence rate as the vanilla SGD method.
\cite{gholami2022trusted} also learn the trusted DFL framework on the limited communications.
\cite{hashemi2021benefits,shi2023improving} verify that DFL suffers from the consensus and may be improved by multi-gossip.
\cite{li2023dfedadmm} propose the adaptation of the variant of the primal-dual optimizer in the DFL framework.
However, experiments in DFL are generally unsatisfactory. Research on its generalization has gradually become one of the hot topics.
\cite{sun2021stability} provide the uniform stability analysis of the decentralized approach and indicate that it could be dominated by the spectrum coefficient. 
\cite{zhu2023decentralized} prove the decentralized approach may be asymptotically equivalent to the SAM optimizer with flat loss landscape and higher generality. 
Different from the previous work, our study mainly focuses on providing a clear and specific answer to the \textit{question} we ask in Section~\ref{introduction}. Meanwhile, our analysis helps to understand whether a topology is suitable for DFL, and how to choose the most suitable training mode under existing conditions.

%% file: texts/methodology.tex
\section{Problem Formulation}
\label{methodology}
\begin{algorithm}[t]
\renewcommand{\algorithmicrequire}{\textbf{Input:}}
\renewcommand{\algorithmicensure}{\textbf{Output:}}
\caption{{\ttfamily FedAvg} Algorithm}
\label{algorithm:fedavg}
\begin{algorithmic}[1]
    \REQUIRE initial model $w^0$,\,$T$,\,$K$,\,$\eta$.
    \ENSURE optimized global model $w^T$.
    \FOR{$t = 0, 1, 2, \cdots, T-1$}
    \STATE randomly select a subset $\mathcal{N}$ from $\left[m\right]$
    \FOR {client $i \in \mathcal{N}$ in parallel}
    \STATE send $w^{t}$ to the client $i$ as $w_{i,0}^t$
    \STATE $w_{i,K}^{t}\leftarrow\textit{SGD-Opt}(w_{i,0}^{t},\eta,K)$
    \STATE send the $w_{i,K}^{t}$ to the server
    \ENDFOR
    \STATE $w^{t+1}\leftarrow\frac{1}{n}\sum_{i\in \mathcal{N}} w_{i,K}^{t}$
    \ENDFOR
\end{algorithmic}
\end{algorithm}

\paragraph{Notations} We first introduce some notations adopted in our paper as follows. Unless otherwise specified, we use italics for scalars, e.g. $n$, and capital boldface for matrix, e.g. $\mathbf{M}$. $\left[n\right]$ denotes a sequence of positive integers from $1$ to $n$. $\mathbb{E}[\cdot]$ denotes the expectation of $\cdot$ term with respect to the potential probability spaces. $\Vert\cdot\Vert$ denotes the $l_2$-norm of a vector and the Frobenius norm of a matrix. $\Vert\cdot\Vert_{op}$ denotes the spectral norm of a matrix. $\vert\cdot\vert$ denotes the absolute value of a scalar. Unless otherwise specified, all four arithmetic operators conform to element-wise operations.

\paragraph{Fundamental Problems} We formulate the fundamental problem as minimizing a finite-sum problem with privacy on each local heterogeneous dataset. We suppose the following simplified scenario: a total of $m$ clients jointly participate in the training whose indexes are recorded as $i$ where $i\in\left[m\right]$. On each client $i$, there is a private dataset $\mathcal{S}_i$ with $S$ data samples. Each sample is denoted as $z_{i,j}$ where $j\in\left[S\right]$. Each dataset $\mathcal{S}_i$ follows a different and independent distribution $\mathcal{D}_i$. We consider the population risk minimization on the finite-sum problem of non-convex objectives $f_i(w,z)$:
\begin{equation}
\label{population_risk}
    \min_w F(w)\triangleq\frac{1}{m}\sum_{i\in\left[m\right]}F_i(w),\quad F_i(w)\triangleq \mathbb{E}_{z\sim\mathcal{D}_i}f_i(w,z),
\end{equation}
where $F(w):\mathbb{R}^d\rightarrow\mathbb{R}$ is denoted as the global objective with respect to the parameters $w$. In the practical cases, we use the surrogate empirical risk minimization~(ERM) objective to replace Eq.(\ref{population_risk}):
\begin{equation}
\label{empirical_risk}
    \min_w f(w)\triangleq\frac{1}{m}\sum_{i\in\left[m\right]}f_i(w),\quad f_i(w)\triangleq \frac{1}{S}\sum_{z\in\mathcal{S}_i}f_i(w,z).
\end{equation}

\begin{algorithm}[t]
\renewcommand{\algorithmicrequire}{\textbf{Input:}}
\renewcommand{\algorithmicensure}{\textbf{Output:}}
\caption{{\ttfamily D-FedAvg} Algorithm}
\label{algorithm:D-FedAvg}
\begin{algorithmic}[1]
    \REQUIRE initial models $w_{i,K}^{-1}$,\,$T$,\,$K$,\,$\eta$, $\mathcal{G}$.
    \ENSURE optimized global model $w^T$.
    \FOR{$t = 0, 1, 2, \cdots, T-1$}
    \FOR {client $i \in \mathcal{N}$ in parallel}
    \STATE send $w_{i,K}^{t-1}$ to its neighbors
    \STATE $w_{i,0}^{t}\leftarrow\sum_{j\in\mathcal{A}_i}a_{ij}w_{i,K}^{t-1}$
    \STATE $w_{i,K}^{t}\leftarrow\textit{SGD-Opt}(w_{i,0}^{t},\eta,K)$
    \STATE send the $w_{i,K}^{t}$ to the server
    \ENDFOR
    \STATE $w^{t+1}\leftarrow\frac{1}{n}\sum_{i\in \mathcal{N}} w_{i,K}^{t}$
    \ENDFOR
\end{algorithmic}
\end{algorithm}

\quad

\paragraph{Centralized FL} CFL employs a global server to coordinate several local clients to collaboratively train a global model. To alleviate the communication costs, it randomly activates a subset $\mathcal{N}\left(\vert\mathcal{N}\vert=n\right)$ among all clients. At the beginning of each round, the global server sends the global model to the active clients as the initialization state. Then each local client performs SGD-Opt process~(total $K$ stochastic gradient descent steps). After the local training, the optimized local models will be sent to the global server for aggregation to generate the global model in the next round, and continue to participate in training until it is well optimized. Algorithm~\ref{algorithm:fedavg} shows the classical {\ttfamily FedAvg} method~\cite{mcmahan2017communication}.

\paragraph{Decentralized FL.} DFL allows each local client to only communicate with its neighbors on an undirected graph $\mathcal{G}$, which is defined as a collection of clients and connections between clients $\mathcal{G}=\left(\mathcal{I},\mathcal{E}\right)$. $\mathcal{I}$ denotes the clients' set $\left[m\right]$ and $\mathcal{E}\subseteq\mathcal{I}\times\mathcal{I}$ denotes the their connections. The relationships are associated with an adjacent matrix $\mathbf{A}=\left[a_{ij}\right]\in\mathbb{R}^{m\times m}$. If $\left(i,j\right)\in\mathcal{E}$, the corresponding element $a_{ij} > 0$, otherwise $a_{ij}=0$. In decentralized setups, all clients aggregate the models within their neighborhoods and perform the SGD-Opt process~(same as above). Algorithm~\ref{algorithm:D-FedAvg} shows the classical {\ttfamily D-FedAvg}~\cite{sun2022decentralized}~(communication-first).

\begin{definition}[Adjacent Matrix]
\label{def_adjacent_matrix}
    The adjacent matrix $\mathbf{A}=\left[a_{ij}\right]\in\mathbb{R}^{m\times m}$ satisfies the following properties: (1) non-negative: $a_{ij}\geq 0$; (2) symmetry: $\mathbf{A}^\top=\mathbf{A}$; (3) $null\left\{\mathbf{I}-\mathbf{A}\right\}=span\left\{1\right\}$; (4) spectral: $\mathbf{I}\succeq\mathbf{A}\succ -\mathbf{I}$; (5) double stochastic: $\mathbf{1}^\top \mathbf{A}=\mathbf{1}^\top, \mathbf{A}\cdot \mathbf{1}=\mathbf{1}$ where $\mathbf{1}=[1,1,1,\cdots,1]^\top$. The eigenvalues $\lambda_i$ of matrix $\mathbf{A}$ satisfies $1=\lambda_1>\lambda_2>\cdots>\lambda_m>-1$, where $\lambda_i$ denotes the $i$-th largest eigenvalue of $\mathbf{A}$. By defining $\lambda\triangleq\max\left\{\vert\lambda_2\vert, \vert\lambda_m\vert\right\}>0$, we can bound the spectral gap of the adjacent matrix $\mathbf{A}$ by $\lambda$.
\end{definition}

\begin{lemma}[\cite{montenegro2006mathematical}]
\label{adjacent_matrix2}
    Let the matrix $\mathbf{P}=\mathbf{1}\mathbf{1}^\top/m\in\mathbb{R}^{m\times m}$, given a positive $t\in\mathbb{Z}^+$, the adjacent matrix $\mathbf{A}$ satisfies $\Vert\mathbf{A}^t-\mathbf{P}\Vert_{\text{op}}\leq\lambda^t$, which measures the ability of aggregation with the adjacent matrix $\mathbf{A}$.
\end{lemma}

\textbf{Generalization Gap.} 
To comprehensively explore the stability in both centralized and decentralized FL, motivated by the previous studies~\cite{hardt2016train,sun2021stability,zhou2021towards,fedinit,sun2023understanding}, we adopt the uniform stability analysis.
\begin{definition}[Uniform Stability]
\label{stability}
    We construct a new joint dataset $\widetilde{\mathcal{C}}$ which only differs from the vanilla dataset $\mathcal{C}$ at most one data sample on the local dataset $\mathcal{S}_i^\star$. Then we say $\mathcal{A}$ is an $\epsilon$-uniformly stable algorithm if:
    \begin{equation}
    \sup_{z\sim\cup\mathcal{D}_i}\mathbb{E}\left[f(\mathcal{A}(\mathcal{C}),z) - f(\mathcal{A}(\widetilde{\mathcal{C}}),z)\right]\leq \epsilon.
    \end{equation}
\end{definition}
\begin{lemma}[\cite{elisseeff2005stability,hardt2016train}]
If a stochastic method $\mathcal{A}$ is $\epsilon$-uniformly stable, we could bound its generalization error as $\varepsilon_G\leq\epsilon$.
\end{lemma}

%% file: texts/theoretical_analysis.tex
\section{Stability and Generalization}
\label{theoretical analysis}

In this part, we mainly introduce the theoretical analysis of the generalization error bound and provide a comparison between centralized and decentralized setups in FL paradigms. We first introduce the main assumptions adopted in this paper and discuss their applicability and our improvements compared with previous studies. Then we state the main theorems, corollaries, and discussions.

\begin{assumption}
\label{assumption:smooth}
    For $\forall \ w_1, w_2\in\mathbb{R}^d$, the objective $f_i(w)$ is $L$-smooth for arbitrary inputs:
    \begin{equation}
        \Vert \nabla f_i(w_1) - \nabla f_i(w_2)\Vert\leq L\Vert w_1 - w_2\Vert.
    \end{equation}
\end{assumption}
\begin{assumption}
\label{assumption:stochastic}
    For $\forall \ w\in\mathbb{R}^d$, the local stochastic gradient $g_i=\nabla f_i(w, z)$ where $z\in\mathcal{S}_i$ is an unbiased estimator of the full local gradient $\nabla f_i(w)$ with a bounded variance:
    \begin{equation}
        \mathbb{E}_z[\ g_i - \nabla f_i(w)] = 0, \ \mathbb{E}_z\Vert g_i - \nabla f_i(w)\Vert^2 \leq \sigma_l^2.
    \end{equation}
\end{assumption}
\begin{assumption}
\label{assumption:lipschitz}
    For parameters $\mathcal{A}\left(\mathcal{C}\right), \mathcal{A}(\widetilde{\mathcal{C}})\in\mathbb{R}^d$ which are well trained by an $\epsilon$-uniformly stable algorithm $\mathcal{A}$ on dataset $\mathcal{C}$ and $\widetilde{\mathcal{C}}$ respectively, the global objective satisfies:
    \begin{equation}
        \vert f(\mathcal{A}(\mathcal{C})) - f(\mathcal{A}(\widetilde{\mathcal{C}})) \vert \leq G\Vert\mathcal{A}(\mathcal{C}) - \mathcal{A}(\widetilde{\mathcal{C}})\Vert.
    \end{equation}
\end{assumption}
\textbf{Discussions.} Assumption~\ref{assumption:smooth} and \ref{assumption:stochastic} are two general assumptions that are widely adopted in the analysis of federated stochastic optimization~\cite{reddi2020adaptive,karimireddy2020scaffold,gorbunov2021local,yang2021achieving,xu2021fedcm,gong2022fedadmm,qu2022generalized,sun2023efficient,huang2023fusion}. Assumption~\ref{assumption:lipschitz} is a variant of the vanilla Lipschitz continuity assumption. The vanilla Lipschitz continuity is widely used in the uniform stability analysis~\cite{elisseeff2005stability,hardt2016train,zhou2021towards,sun2021stability,xiao2022stability,zhu2022topology,fedinit}, which implies the objective have bounded gradients $\Vert\nabla f(w)\Vert\leq G$ for $\forall \ w\in\mathbb{R}^d$. However, several recent works have shown that it may not always hold in current deep learning~\cite{kim2021lipschitz,mai2021stability,patel2022gradient,das2023beyond}. Therefore, in order to improve the applicability of the stability analysis on general deep models, we use Assumption~\ref{assumption:lipschitz} instead, which could be approximated as a specific Lipschitz continuity only at the minimum $\mathcal{A}(\mathcal{C})$. The main challenge without the strong assumption is that the boundedness of the iterative process cannot be ensured. Our proof indicates that even if the iterative process is not necessarily bounded, the uniform stability can still maintain the vanilla upper bound.

\subsection{Stability of CFL~(Algorithm~\ref{algorithm:fedavg})}
\begin{theorem}
\label{thm1}
    Under Assumption~\ref{assumption:smooth}$\sim$~\ref{assumption:lipschitz}, let the active ratio per communication round be $n/m$, and let the learning rate $\eta=\mathcal{O}\left(\frac{1}{tK+k}\right)=\frac{\mu}{tK+k}$ is decayed per iteration $\tau=tK+k$, where $\mu$ is a specific constant which satisfies $\mu \leq \frac{1}{L}$, let $U$ be the maximization of loss value, we have:
    \begin{equation}
        \varepsilon_G
        \leq \frac{2\sigma_l G}{mSL}\left(\frac{TK}{\tau_0}\right)^{\mu L}+ \frac{nU\tau_0}{mS}.
    \end{equation}
    By selecting a proper $\tau_0=\left(\frac{2\sigma_l G}{nUL}\right)^\frac{1}{1+\mu L}\left(TK\right)^{\frac{\mu L}{1+\mu L}}$, we can minimize the generalization error bound:
    \begin{equation}
        \varepsilon_G
        \leq \frac{4}{S}\left(\frac{\sigma_l G}{L}\right)^\frac{1}{1+\mu L}\left(\frac{n^\frac{\mu L}{1+\mu L}}{m}\right)\left(UTK\right)^\frac{\mu L}{1+\mu L}.
    \end{equation}
\end{theorem}

\begin{remark}[Stability.]
    The generalization error in the CFL is mainly affected by the number of samples $S$, the number of total clients $m$, the number of active clients $n$, and the total iterations $TK$. The vanilla SGD achieves the upper bound of $\mathcal{O}\left(T^\frac{\mu L}{1+\mu L}/S\right)$~\cite{hardt2016train}, while the stability of CFL is worse than vanilla SGD due to additional negative impacts of $n$ and $K$. Specifically, when the number of active clients and local intervals increases, its performance will decrease significantly. An intuitive understanding is that when the number of active clients increases, it will be easier to select new samples that are not consistent with the current understanding. This results in the model having to adapt to the new knowledge.
\end{remark}

\begin{remark}[Best participation.]
\label{FL_excess_risk}
    From the generalization error perspective, it achieves the best stability when $n=1$. Under this, CFL degrades to vanilla SGD. However, this is undoubtedly significantly inefficient for optimization. Current studies in FL have proven that convergence maintains the linear speedup property~\cite{yang2021achieving}, where enlarging $n$ could improve the convergence rate. Therefore, it is a trade-off for CFL in selecting the proper $n$. To balance optimization and generalization errors and get lower test error, the optimal $n$ is {\ttfamily partial}-participation training.
\end{remark}

When $n=1$, CFL degenerates into the conclusion of the vanilla SGD. When $m=n$ and $K>1$, it degenerates into the conclusion of the local SGD. The analysis in centralized federated learning requires the $m \geq n >1$ and $K > 1$. In Remark~\ref{FL_excess_risk}, our analysis points out it is a trade-off on $n$, $K$, and $T$ in {\ttfamily FedAvg} and centralized federated learning. We meticulously provide the recommended selection for the number of active clients $n$ to achieve optimal efficiency.

Generally, the communication round $T$ is decided by the training costs and local computing power. Therefore, under a fixed local interval $K$, when the optimization convergence dominates the test accuracy, increasing $n$ brings linear speedup property and effectively accelerates the training process~\cite{yang2021achieving}. Most of the advanced federated methods have also been proven to benefit from this speedup. However, when the generalization error dominates the test accuracy, large $n$ is counterproductive. \cite{charles2021large} similarly observe the experiments that selecting a very large cohorts~($n$) generally leads to lower accuracy. We theoretically prove this phenomenon and provide a rough estimation of the best value of active ratios. Let the linear speedup be $\mathcal{O}\left(1 \ /\ n\right)$~\cite{yang2021achieving} in the optimization and $\mathcal{O}\left(n^{\frac{\mu L}{1+\mu L}}\ /\ m\right)$ in the generalization, the optimal $n$ satisfies $n^\star\approx\mathcal{O}(m^\frac{1+\mu L}{1+2\mu L})$, which could optimally balance the optimization and generalization error.



\subsection{Stability of DFL~(Algorithm~\ref{algorithm:D-FedAvg})}

\begin{theorem}
\label{thm2}
    Under Assumption~\ref{assumption:smooth}$\sim$~\ref{assumption:lipschitz}, let the communication graph be $\mathbf{A}$ as introduced in Definition~\ref{def_adjacent_matrix} which satisfies the conditions of spectrum gap $\lambda$ in Lemma~\ref{adjacent_matrix2}, and let the learning rate $\eta=\mathcal{O}\left(\frac{1}{tK+k}\right)=\frac{\mu}{tK+k}$ is decayed per iteration $\tau=tK+k$ where $\mu$ is a specific constant which satisfies $\mu \leq \frac{1}{L}$, let $U$ be the maximization of loss value, we have:
    \begin{equation}
        \varepsilon_G
        \leq \frac{2\sigma_l G}{SL}\left(\frac{1+6\sqrt{m}\kappa_\lambda}{m}\right)\left(\frac{TK}{\tau_0}\right)^{\mu L} + \frac{U\tau_0}{S}.
    \end{equation}
    where $\kappa_\lambda$ is a constant~(Lemma~\ref{lambda}).
    By selecting a proper $\tau_0=\left(\frac{2\sigma_l G}{UL}\frac{1+6\sqrt{m}\kappa_\lambda}{m}\right)^\frac{1}{1+\mu L}\left(TK\right)^{\frac{\mu L}{1+\mu L}}$, we can minimize the generalization error bound as:
    \begin{equation}
        \varepsilon_G
        \leq \frac{4}{S}\left(\frac{\sigma_l G}{L}\right)^\frac{1}{1+\mu L}\left(\frac{1+6\sqrt{m}\kappa_\lambda}{m}\right)^{\frac{1}{1+\mu L}}\left(UTK\right)^{\frac{\mu L}{1+\mu L}}.
    \end{equation}
\end{theorem}

\begin{remark}[Stability.]
\label{DFL_stability}
    The generalization error in DFL performs with the impact of the number of samples $S$, the number of total clients $m$, and total iterations $TK$. Differently, it is also affected by the topology. $\kappa_\lambda$ is also a widely used coefficient related to the $\lambda$ that could measure different connections in the topology. Its stability achieves the best performance when we select the $\kappa_\lambda=0$, which corresponds to the {\ttfamily full}-topology. In practical scenarios, DFL prefers a small $\kappa_\lambda$ coefficient to improve generalization performance as much as possible, which indicates that the links in the communication topology should be dense enough.
\end{remark}

\begin{remark}[Best Topology.]
\label{DFL_excess_risk}
    Similarly, we need to comprehensively consider the joint performance of the optimization and generalization in DFL. \cite{haddadpour2019convergence,sun2022decentralized} study the convergence and provide the analysis that the spectrum gap $\lambda$ usually exists in the non-dominant terms in DFL. To minimize the impacts of the spectrum gap in convergence rate, it still prefers a dense topology. Therefore, without considering communication costs, {\ttfamily full}-topology is the best selection in DFL which minimizes both the optimization and generalization error.
\end{remark}

\begin{table}[t]
  \caption{Comparison of common topologies. The arrow denotes the trends as $m$ increases.}
  \label{DFL_kappa}
  \centering
  \vspace{-0.2cm}
  
  \begin{tabular}{ccccc}
    \toprule
    Topologies     &   $\kappa_\lambda$   &  $\varepsilon_G$ & Trends \\
    \midrule
    {\ttfamily full}    & 0 & $\widetilde{\mathcal{O}}\left(m^{-\frac{1}{1+\mu L}}\right)$ & $\downarrow$\\
    \specialrule{0em}{1pt}{1pt}
    {\ttfamily exp}    & $\mathcal{O}\left(\ln m\right)$ & $\widetilde{\mathcal{O}}\left(m^{-\frac{1}{2(1+\mu L)}}\right)$ & $\downarrow$ \\
    \specialrule{0em}{1pt}{1pt}
    \midrule
    \specialrule{0em}{1pt}{1pt}
    {\ttfamily grid}            & $\mathcal{O}\left(m\ln m\right)$ & $\widetilde{\mathcal{O}}\left(m^{\frac{1}{2(1+\mu L)}}\right)$ & $\uparrow$ \\
    \specialrule{0em}{1pt}{1pt}
    {\ttfamily ring}            & $\mathcal{O}\left(m^2\right)$  & $\widetilde{\mathcal{O}}\left(m^{\frac{3}{2(1+\mu L)}}\right)$ & $\uparrow$ \\
    \specialrule{0em}{1pt}{1pt}
    {\ttfamily star}            & $\mathcal{O}\left(m^2\right)$  & $\widetilde{\mathcal{O}}\left(m^{\frac{3}{2(1+\mu L)}}\right)$ & $\uparrow$ \\
    \bottomrule
  \end{tabular}
\end{table}

Remark~\ref{DFL_excess_risk} explains that the spectrum gap mainly affects the generalization error. Similarly, we consider the communication rounds $T$ to be determined by local computing power. When the local interval $K$ is fixed, selecting the topology with larger $\kappa_\lambda$ will achieve a bad performance. As stated in Remark~\ref{DFL_stability}, when $\kappa_\lambda=0$ it achieves a minimal upper bound. This also indicates that the {\ttfamily full}-topology is the best selection. According to the research of \cite{ying2021exponential,zhu2023stability}, we summarize some $\lambda$ and $\kappa_\lambda$ values of classical topologies in Table~\ref{DFL_kappa} to compare their generalization performance. Our theoretical analysis indicates that the classical topologies, i.e. {\ttfamily grid}, {\ttfamily ring}, and {\ttfamily star}, will lead to a significant drop in stability as total clients $m$ increase in DFL. In contrast, the {\ttfamily full} and {\ttfamily exp} topologies show better performance which can reduce the generalization error as $m$ increases. In FL setups, larger $m$ also means more data samples. That is, if the topology is not good enough, i.e., $\kappa_\lambda > \mathcal{O}(\sqrt{m})$, even if we feed more training data samples, it still gets worse stability in DFL. Generally, small $\kappa_\lambda$ always means more connections in the topology and more communication costs, which is similar to the spectrum gap $1-\lambda$ being large enough~($\lambda$ is small enough). This also gives us a clear insight to design topologies. We show some classical designs of fixed topologies in the Appendix.

\begin{table*}[t]
  \caption{Comparisons of current theoretical analysis in decentralized methods on the non-convex objectives. ``BG" means the bounded gradient, i.e. $\Vert \nabla f(w)\Vert^2\leq G^2$, which is a strong assumption of the Lipschitz continuous functions. ``BL" means the bounded loss, i.e. $\sup f < +\infty$. ``BV" means the bounded gradient variance, i.e. Assumption~\ref{assumption:stochastic} in our paper. $\prod_\eta$ is the polynomial of $\eta_t$.}
  \label{DFL_res_comp}
  \centering
  \vspace{-0.2cm}
  
  \begin{tabular}{cccccc}
    \toprule
         &  Method  &  Assumption   &  Results  &  Scenarios  \\
    \midrule
    \cite{sun2021stability} & {\ttfamily D-SGD} & smooth, BG, BL & $\mathcal{O}\left(\frac{U(1+\kappa_\lambda)}{Sm}T^{\frac{\mu L}{1+\mu L}}\right)$ & \multirow{6}{*}{IID dataset, $K=1$} \\
    \specialrule{0em}{2pt}{2pt}
    \cite{belletimproved} & {\ttfamily D-SGD} & smooth, BG, BL & $\mathcal{O}\left(\frac{U^{\frac{\mu L}{1 + \mu L}}}{Sm^{\frac{1}{1+\mu L}}}T^{\frac{\mu L}{1 + \mu L}}\right)$ & \\
    \specialrule{0em}{2pt}{2pt}
    \cite{wang2024towards} & {\ttfamily D-SGD-ZO} & smooth, BG, BL & $\mathcal{O}\left(\frac{U(1+C_d)}{Sm}\prod_\eta\right)$ & \\
    \specialrule{0em}{2pt}{2pt}
    \hline
    \specialrule{0em}{2pt}{2pt}
    \cite{zhu2024stability} & {\ttfamily D-SGDA} & smooth, BG, BL & $\mathcal{O}\left(\left(\frac{m}{S}\right)^{\frac{1}{\mu + L}}\kappa_\lambda^{\frac{1}{\mu + L}}T^{1-\frac{1}{\mu + L}}\right)$ & nonconvex-nonconcave\\
    \specialrule{0em}{2pt}{2pt}
    \hline
    \specialrule{0em}{2pt}{2pt}
    \cite{hua2022efficient} & {\ttfamily D-FedAvg} & smooth, BG, BL, BV & $\mathcal{O}\left(\frac{UK}{S}(TK)^{\frac{\mu L}{1+\mu L}}\right) + \mathcal{O}\left(\frac{1}{L}\right)$ & \multirow{4}{*}{non-IID dataset, $K>1$} \\
    \specialrule{0em}{2pt}{2pt}
    ours & {\ttfamily D-FedAvg} & smooth, BL, BV & $\mathcal{O}\left(\frac{U^{\frac{\mu L}{1+\mu L}}}{S}\left(\frac{1+6\sqrt{m}\kappa_\lambda}{m}\right)^{\frac{1}{1+\mu L}}\left(TK\right)^{\frac{\mu L}{1+\mu L}}\right)$ & \\
    \bottomrule
  \end{tabular}
\end{table*}

\begin{remark}[Theoretical comparisons.]
We compare our results with some existing decentralized theoretical results in Table~\ref{DFL_res_comp}. \cite{sun2021stability} provide a general generalization analysis for the {\ttfamily D-SGD} method, where a stability gap recurrence relation was constructed via a coefficient related to the spectral gap in the aggregation process, leading to the derivation of an upper bound on the generalization error. \cite{belletimproved} further improve this error bound and demonstrate that, under homogeneous data, the bound could recover to a limit independent of the topology parameters (asymptotically approaching SGD). \cite{wang2024towards} and \cite{zhu2024stability} additionally provide conclusions about zero-order {\ttfamily D-SGD} and {\ttfamily D-SGDA}. \cite{hua2022efficient} focus on the heterogeneous dataset and provide an error bound asymptotically approaching SGD based on \cite{sun2021stability}. However, it does not successfully capture the subtle effects of the client scale $m$ and the impact of the topology. Our results are an improved version of the conclusions which can directly degrade to results in \cite{belletimproved} by select $\kappa_\lambda=0$ and $K=1$. Meanwhile, it also preserves the variation in topology as a result of \cite{sun2021stability} without being affected by an additional term $K$. Our theoretical analysis also eliminates the dependence on the BG~(bounded gradient) assumption, which can recover the vanilla error bound in a broader scope.
\end{remark}

\subsection{Comparisons between CFL \& DFL}
In this section, we mainly discuss the strengths and weaknesses of each framework. In Table~\ref{Comparison_FL_DFL}, we summarize Remark~\ref{FL_excess_risk} and \ref{DFL_excess_risk} and their optimal selections with impacts of the $S$, $m$, $T$, and $K$ terms.
$n$ could be considered as the ``topology'' in CFL and $\kappa_\lambda$ represents the impacts from the topology in DFL. Therefore, to understand which mode is better, we can directly study the impacts of their connections.
In Table~\ref{Comparison_FL_DFL}, we could know {\ttfamily FedAvg} always generalizes better than {\ttfamily D-FedAvg}, which largely benefits from regularly averaging on a global server. In the whole training process, centralized approaches always maintain a high global consensus. Though \cite{lian2017can} have learned that the computational complexity of the C$/$D-PSGD may be similar in the optimization, from the perspective of test error and excess risk, CFL shows stronger stability and more excellent generalization ability. 
However, the high communication costs in centralized approaches nevertheless are unavoidable. The communication bottleneck is one of the important concerns restricting the development of federated learning. To achieve reliable performance, the number of active clients $n$ in CFL must satisfy at least a polynomial order of $m$. Too small $n$ will always hurt the performance of the optimization. In the DFL framework, communication is determined by the average degree of the adjacent matrix. For instance, in the exponential topology, the communication costs achieve $\mathcal{O}\left(\log m\right)$ at most in one client~\cite{ying2021exponential}, which is less than the communication overhead in CFL. Therefore, when we have to consider the communication bottleneck, it is also possible to select DFL at the expense of generalization performance.

\quad

We simply summarize the setups we considered as follows. We can assume that: (1) the communication capacity of the global server in CFL is $\rho\times$ larger than the common clients; (2) each local client can support at most $N_D$ connections simultaneously. If $m\leq\left(\rho N_D\right)^\frac{1+2\mu L}{1+\mu L}\leq \left(\rho N_D\right)^{1.5}$, the centralized approaches definitely performs better. When $m$ is very large, although DFL can save communications as a compromise, its generalization performance will be far worse than CFL.

\begin{remark}
In general non-convex FL setups, the lower bound of the stability is still an open problem worthy further in-depth studies. However, current studies on the general SGD method have revealed when $T$ is large enough, the upper bound can be sufficiently sharp~\cite{zhang2022stability,zhang2023lower}. Specifically, the lower bound of SGD achieves at least $\mathcal{O}\left(\frac{T^{c}}{S^{1+c}}\right)$ where $c < 1$ and when $T^{\frac{c}{1+c}}\geq S$ its lower bound will achieve close to the upper bound, yielding an accurate analysis. Due to the complexity of the FL scenario, its lower bound must not be superior to the vanilla SGD method. This also illustrates that our analysis could characterize their properties when $T$ is sufficiently large.
\end{remark}

%% file: texts/experiments.tex
\section{Experiments}
\label{experiments}

\begin{figure*}[t]
        \centering
        \subfloat[Ring.]{
            \includegraphics[width=0.19\textwidth]{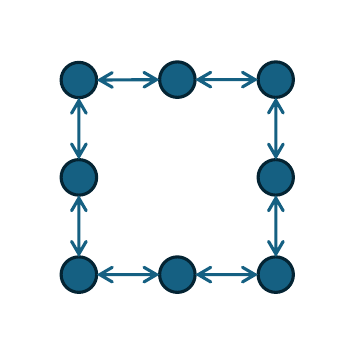}
            \label{ap:ring}
        }\!\!\!\!\!\!\!
        \subfloat[Grid.]{
            \includegraphics[width=0.19\textwidth]{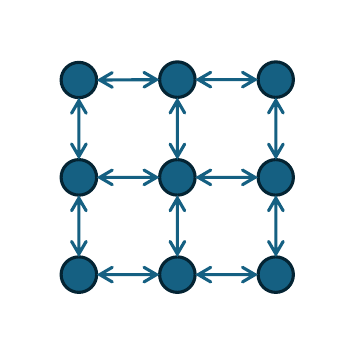}
            \label{ap:grid}
        }\!\!\!\!\!\!\!
        \subfloat[Star.]{
            \includegraphics[width=0.19\textwidth]{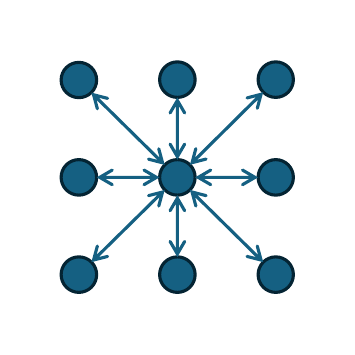}
            \label{ap:star}
        }\!\!\!\!\!\!\!
        \subfloat[Exp.]{
            \includegraphics[width=0.19\textwidth]{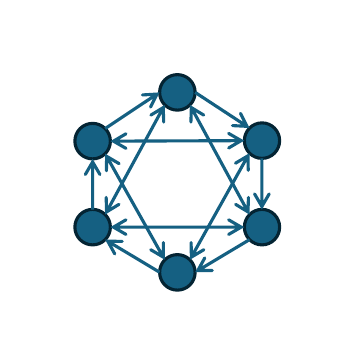}
            \label{ap:exp}
        }\!\!\!\!\!\!\!
        \subfloat[Full.]{
            \includegraphics[width=0.19\textwidth]{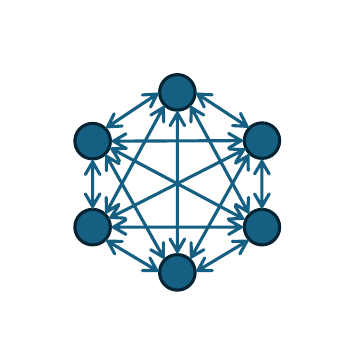}
            \label{ap:full}
        }
        \caption{Some classical topologies in the decentralized approaches.}
        \label{ap:topologies}
    \end{figure*}

In this section, we primarily introduce empirical studies. We first introduce the general setups including the datasets, models, and hyperparameter selections. Then we introduce the main experiments in CFL and DFL to validate our theoretical analysis. Details are stated in the Appendix.

\quad

\paragraph{Datasets} We conduct empirical studies on the image classification task on CIFAR-10 dataset~\cite{cifar100}. Our experiments focus on the validation of the theoretical analysis above and we follow \cite{dirichlet} to split the dataset with a Dirichlet distribution, which is widely used in the FL community. We denote Dirichlet-$\beta$ as different heterogeneous levels that are controlled by the concentration parameter $\beta$. In our paper, we select $\beta=0.1$ for the heterogeneity. The number of total clients is selected from $\left[100, 200, 500\right]$.

\quad

\paragraph{Models} ResNet~\cite{resnet} is a classical and fundamental backbone in the studies of federated learning. Most previous works have performed validation experiments based on this model. However, there are various fine-tuning structures adopted in different studies which makes the test accuracy claimed in different papers difficult to compare directly. We also try some structural modifications, i.e. using a small convolution size at the beginning, which may improve the performance without any other tricks. In order to avoid errors in reproduction or comparison, we use the implementation in the Pytorch Model Zoo without other handcraft adjustments. The out dimension of the last linear layer is decided by the total classes of the dataset.

\quad

\paragraph{Hyperparameters} We follow previous classical studies to select the hyperparameters in both CFL and DFL scenarios. We fix the total communication rounds are set as $T=1000$ and set weight decay as $0.001$ without momentum. We select local epochs from $\left[5, 20\right]$.

\quad

\begin{table}[t]
    \caption{Optimal selection of batchsize.}
    \label{batchsize}
    \centering
    \vspace{-0.2cm}
    \begin{tabular}{c| c| c| c |c}
    \toprule[1pt]
		  & Total Clients & E & Selection & Optimal\\
		\midrule
		& m\ =\ 500 &   &   & 10\\
            & m\ =\ 200 &  5 &  & 20\\
            & m\ =\ 100 &   &   & 20\\
            CFL & & & $\left[10, 20, 40, 80\right]$ & \\
            
            & m\ =\ 500 &   &   & 20\\
            & m\ =\ 200 &  20 &  & 40\\
            & m\ =\ 100 &   &   & 80\\
            \midrule
            & m\ =\ 500 &   &   & 40\\
            & m\ =\ 200 &  5 &  & 40\\
            & m\ =\ 100 &   &   & 40\\
            DFL & & & $\left[10, 20, 40, 80\right]$ & \\
            
            & m\ =\ 500 &   &   & 40\\
            & m\ =\ 200 &  20 &  & 80\\
            & m\ =\ 100 &   &   & 80\\
		\bottomrule [1pt]
        \end{tabular}
        \vskip -0.1in
    \end{table}

\textbf{Batchsize.} An important metric is the choice of batchsize, which relates to training fairness under the same local training interval in CFL and DFL respectively. Traditional comparisons are based on the number of data passes, i.e. training epochs. This also means that when $K$ and $S$ are the fixed, the batchsize should also be chosen consistently. However, we find in the experiments that the choice of batchsize has its own tendencies under these two modes. Blindly using the optimal choice from one mode as a reference can severely undermine the performance in the other mode. Current works mainly prefer two selections of the batchsize. One is 50~\cite{karimireddy2020scaffold,acar2021federated,sun2023fedspeed,sun2023dynamic}, and the other is 128~\cite{qu2022generalized,shi2023improving}. \cite{xu2021fedcm} also discuss some different selections based on manual adjustments. Motivated by the previous works and the fact that batchsize is affected by the local intervals and the learning rate, we conduct extensive experiments and select the optimal value which could achieve significant performance, as shown in Table~\ref{batchsize}. 

\quad

\textbf{Learning rate.} In the previous studies, they unanimously select the local learning rate as $0.1$. \cite{xu2021fedcm,shi2023improving} also test different selections and confirm the optimal selection as $0.1$. We follow this selection to fairly compare their performance.

\quad

\textbf{Weight Decay.} We test some common selections from $\left[0.01,0.001,0.0005,0.0001\right]$. Its optimal selection jumps between $0.01$ and $0.005$. Even in similar scenarios, there will be some differences in their optimal values. The fluctuation amplitude reflected in the test accuracy is about $0.4\%$. For a fair comparison, we fix it as a median value of $0.001$.

\quad

\paragraph{Topology} Fig.~\ref{ap:topologies} introduces $5$ classical topologies in the decentralized communication. Ring topology only activates the connection between the adjacent clients and it only allows the information communicated around the main branch. Grid topology is much better since it allows the node to be accessed by its neighbours~($1$ hop in the distance), which is more practical in the real-world setups. Star topology is more similar to the centralized approach since a node plays the role of the global server and can connect the other nodes. However, their calculation types are exactly different. Exp topology allows a node to connect approximately $\log(m)$ nodes~(total $m$ nodes) in a graph and it will communicate to the neighbors that are $2^0, 2^1, \cdots, 2^{\log(m)+1}$ hops away. Full topology means all nodes are open to be accessed and it is equal to the full participation in the centralized approaches.

\subsection{Centralized Federated Learning}
\label{exp:fl_active_ratio}

\begin{figure*}[t]
    \centering
    \subfloat[m\ =\ 100, E\ =\ 5]{
        \includegraphics[width=0.32\textwidth]{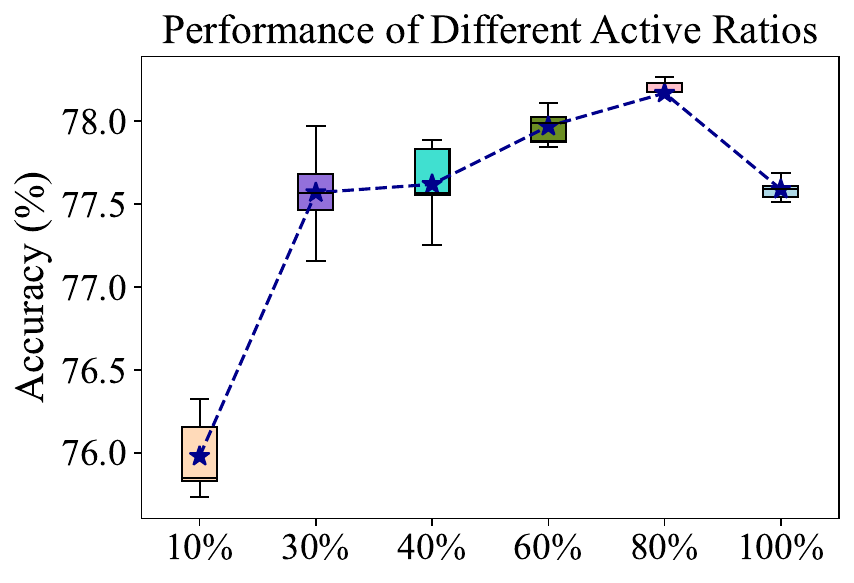}
        \label{cfl-c10-100-E5-active}
    }\!\!\!\!\!
    \subfloat[m\ =\ 200, E\ =\ 5]{
        \includegraphics[width=0.32\textwidth]{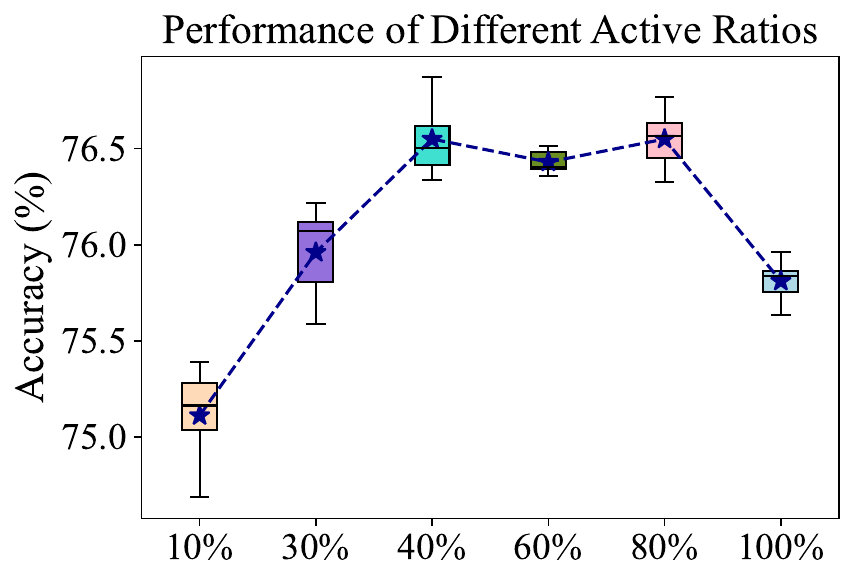}
        \label{cfl-c10-200-E5-active}
    }\!\!\!\!\!
    \subfloat[m\ =\ 500, E\ =\ 5]{
        \includegraphics[width=0.32\textwidth]{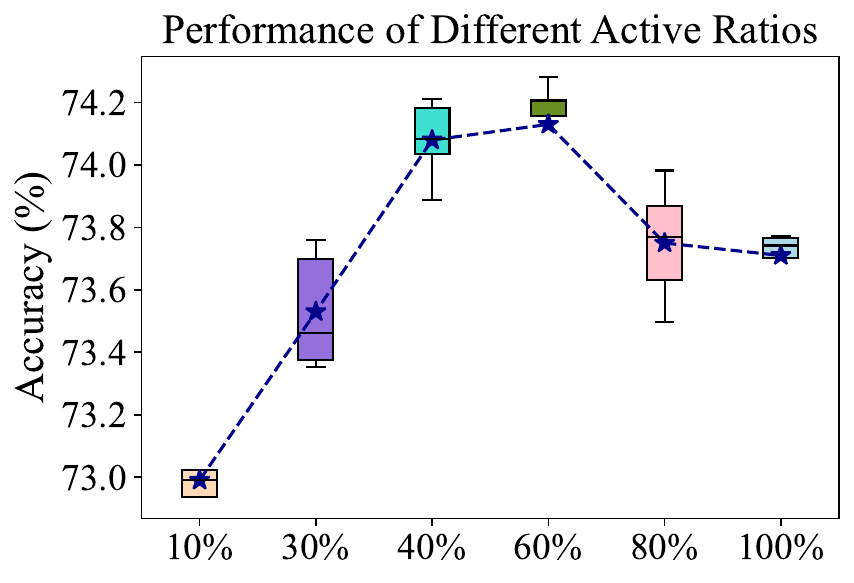}
        \label{cfl-c10-500-E5-active}
    }
    \quad
    \subfloat[m\ =\ 100, E\ =\ 20]{
        \includegraphics[width=0.32\textwidth]{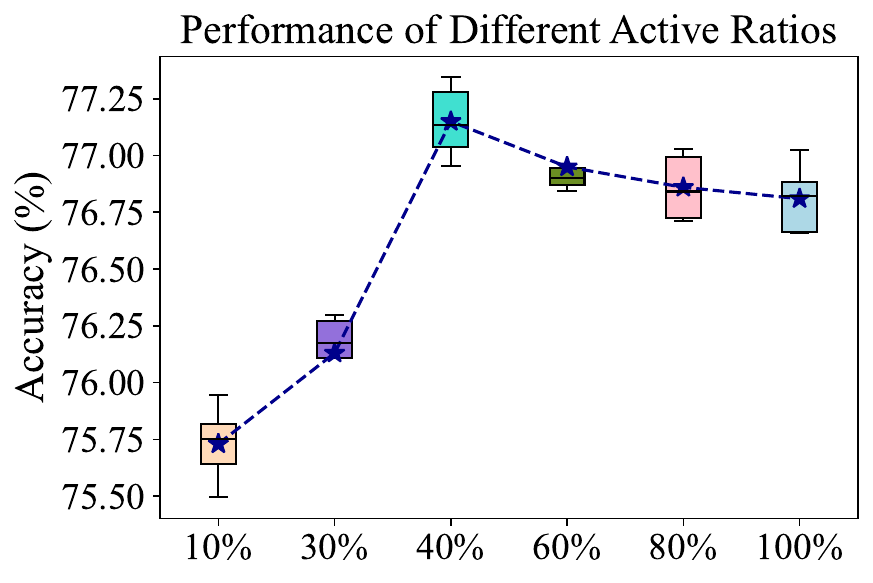}
        \label{cfl-c10-100-E20-active}
    }\!\!\!\!\!
    \subfloat[m\ =\ 200, E\ =\ 20]{
        \includegraphics[width=0.32\textwidth]{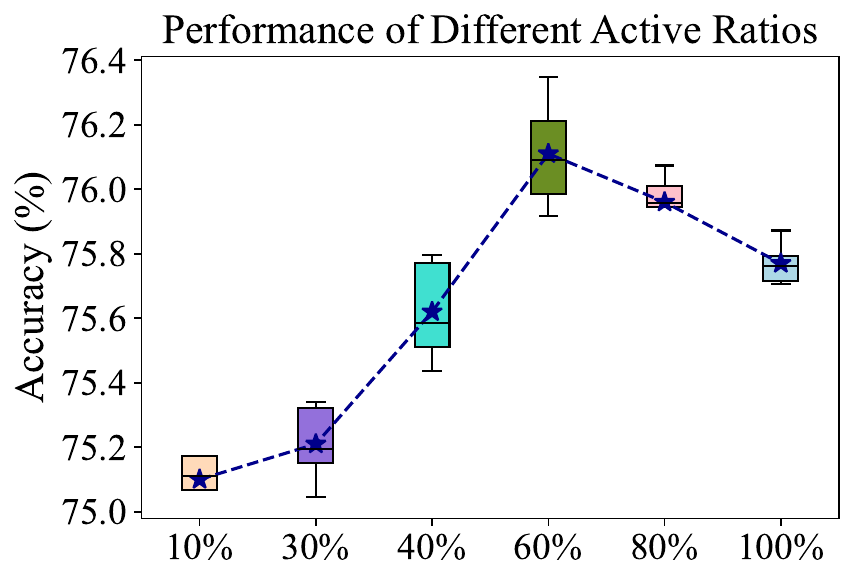}
        \label{cfl-c10-200-E20-active}
    }\!\!\!\!\!
    \subfloat[m\ =\ 500, E\ =\ 20]{
        \includegraphics[width=0.32\textwidth]{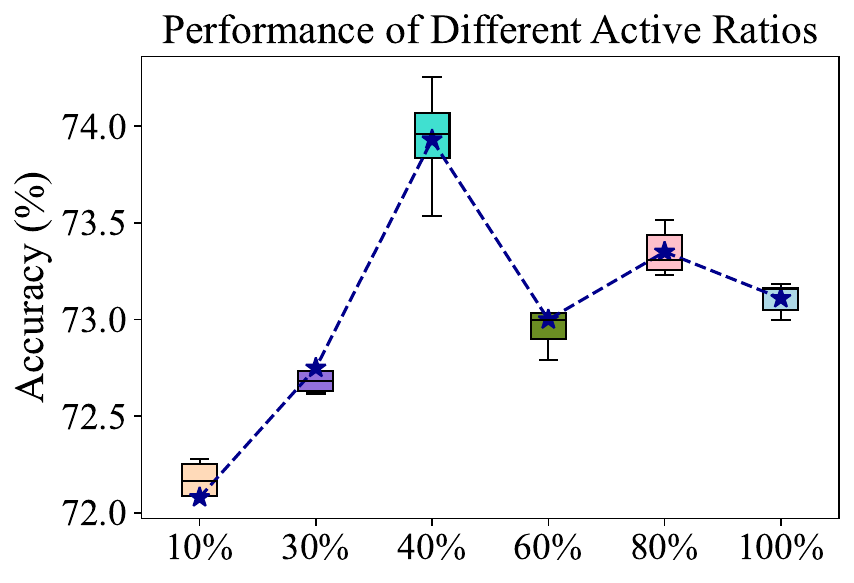}
        \label{cfl-c10-500-E20-active}
    }
    \caption{We test different active ratios in CFL on the CIFAR-10 dataset with the ResNet-18 model. $m$ is the number of clients and $E$ is the number of local epochs. Each setup is repeated 5 times.}
    \label{fl_active_ratio}
\end{figure*}

\textbf{Different Participation Ratios.} Our theoretical analysis in Theorem~\ref{thm1} indicates that active ratios in CFL balance the optimization and generalization errors. To achieve the best performance, it usually does not require all clients to participate in the training process per round. As shown in Figure~\ref{fl_active_ratio}, under the total $100$ clients and local epochs $E=5$, when it achieves the best performance the active ratio is approximately $30\%$. As the number of total clients increases to $500$, the best selection of the active ratio is approximately $40\%$. When this critical value is exceeded, continuing to increase the active ratio causes significant performance degradation. The optimal active ratio is roughly between $40\%$ and $80\%$. When all clients participating in the training as $n=m$, CFL could be considered as DFL on the {\ttfamily full}-connected topology. Therefore, we can intuitively see the poor generalization of decentralized in Figure~\ref{fl_active_ratio}, i.e., full participation is slightly worse than partial participation~\cite{charles2021large}. 

\quad

\textbf{Different Local Intervals.}
According to the Corollary~\ref{FL_excess_risk} and the corresponding discussions, we know the optimal number of the active ratio will decrease as the local interval $K$ increases. Figure~\ref{fl_active_ratio} also validates this in the experiments. When we fix the total clients $m$, we can see that the best performance corresponds to smaller active ratios when local epochs $E$ increase from $5$ to $20$. For instance, when $m=100$, the optimal selection of the active ratio approximately decreases from $80\%$ to $40\%$. And, when the local interval $K$ increases, we can see that test errors increase. As claimed by \cite{karimireddy2020scaffold}, FL suffers from the client-drift problem. When $K$ is large enough, local models will overfit the local optimums and get far away from the global optimum. Due to the space limitation, full curves of the loss and accuracy are stated in Appendix~\ref{ap:fl_local_interval}.

\begin{table*}[t]
\begin{center}
\renewcommand{\arraystretch}{1}
\caption{Comparison on the CFL and DFL. In the CFL setup, we assume that the global server can support $2\times$ the communication capabilities of local clients. {\ttfamily FedAvg-$n$} means the number of active clients equals $n$, which is determined by the corresponding topology in DFL. Each result is repeated $3$ times with different random seeds and smoothed on the last 100 iterations.}
\vspace{0.1cm}
\small
\begin{tabular}{@{}c|cccccc@{}}
\toprule
\multicolumn{1}{c}{} & \multicolumn{2}{c}{$m=100$} & \multicolumn{2}{c}{$m=200$} & \multicolumn{2}{c}{$m=500$} \\ 
\cmidrule(lr){2-3} \cmidrule(lr){4-5} \cmidrule(lr){6-7}
\multicolumn{1}{c}{} & \multicolumn{1}{c}{$E=5$} & \multicolumn{1}{c}{$E=20$} & \multicolumn{1}{c}{$E=5$} & \multicolumn{1}{c}{$E=20$} & \multicolumn{1}{c}{$E=5$} & \multicolumn{1}{c}{$E=20$}\\
\cmidrule(lr){1-1} \cmidrule(lr){2-7}
{\ttfamily D-FedAvg-ring}      & $\text{59.20}_{\pm.15}$ & $\text{46.65}_{\pm.17}$ & $\text{50.14}_{\pm.17}$ & $\text{50.70}_{\pm.12}$ & $\text{41.19}_{\pm.14}$ & $\text{45.20}_{\pm.11}$ \\
{\ttfamily FedAvg-3}       & $\text{67.71}_{\pm.48}$ & $\text{68.96}_{\pm.26}$ & $\text{67.06}_{\pm.33}$ & $\text{66.66}_{\pm.41}$ & $\text{63.33}_{\pm.24}$ & $\text{62.72}_{\pm.47}$ \\ 
{\ttfamily FedAvg-6}       & $\text{74.17}_{\pm.35}$ & $\text{73.63}_{\pm.39}$ & $\text{73.08}_{\pm.31}$ & $\text{72.00}_{\pm.27}$ & $\text{68.05}_{\pm.54}$ & $\text{67.50}_{\pm.34}$ \\ 
\cmidrule(lr){1-1} \cmidrule(lr){2-7}
{\ttfamily D-FedAvg-grid}      & $\text{73.27}_{\pm.19}$ & $\text{73.45}_{\pm.11}$ & $\text{67.70}_{\pm.09}$ & $\text{68.60}_{\pm.17}$ & $\text{58.20}_{\pm.39}$ & $\text{59.36}_{\pm.17}$ \\
{\ttfamily FedAvg-5}       & $\text{73.13}_{\pm.37}$ & $\text{72.97}_{\pm.33}$ & $\text{71.95}_{\pm.26}$ & $\text{70.84}_{\pm.35}$ & $\text{67.10}_{\pm.38}$ & $\text{66.27}_{\pm.55}$ \\ 
{\ttfamily FedAvg-10}       & $\text{75.48}_{\pm.44}$ & $\text{75.03}_{\pm.29}$ & $\text{73.97}_{\pm.18}$ & $\text{73.39}_{\pm.57}$ & $\text{71.06}_{\pm.27}$ & $\text{70.44}_{\pm.41}$ \\ 
\cmidrule(lr){1-1} \cmidrule(lr){2-7}
{\ttfamily D-FedAvg-exp}      & $\text{76.54}_{\pm.11}$ & $\textbf{76.12}_{\pm.08}$ & $\text{74.05}_{\pm.12}$ & $\text{74.43}_{\pm.15}$ & $\text{67.28}_{\pm.12}$ & $\text{68.11}_{\pm.13}$ \\
{\ttfamily FedAvg-$\log m$}      & $\text{75.26}_{\pm.42}$ & $\text{74.42}_{\pm.36}$ & $\text{74.11}_{\pm.29}$ & $\text{73.93}_{\pm.22}$ & $\text{70.04}_{\pm.28}$ & $\text{69.55}_{\pm.34}$ \\ 
{\ttfamily FedAvg-$2\log m$}      & $\textbf{77.19}_{\pm.17}$ & $\text{76.09}_{\pm.23}$ & $\textbf{75.57}_{\pm.26}$ & $\textbf{74.86}_{\pm.17}$ & $\textbf{71.61}_{\pm.34}$ & $\textbf{71.49}_{\pm.28}$ \\ 
\bottomrule
\end{tabular}
\label{acc}
\end{center}
\end{table*}

\subsection{Decentralized Federated Learning}
\label{experiment2}

\textbf{CFL v.s. DFL.} Table~\ref{acc} shows the test accuracy between DFL with different topologies and corresponding CFL. To achieve fair comparisons, we copy the DFL's optimal hyperparameters to the CFL setup and select a proper partial participation ratio in CFL which follows that the communication bandwidth of the global server is $1\times$ or $2\times$ of the local clients. For instance, ring-topology is equivalent to $3$ local clients jointly train one model. Therefore, we test the number of active clients in CFL equals $3$~(1$\times$) and $6$~(2$\times$) respectively. In fact, the global server is much more than twice the communication capacity of local devices in practice. From the results, we can clearly see CFL always generalizes better than the DFL on $2\times$ bandwidth. Even if the global server has the same bandwidth as local clients, DFL is only slightly better than CFL on the $m=100$ setup. With the increase of $m$, DFL's performance degradation is very severe. This is also in line with our conclusions in Table~\ref{Comparison_FL_DFL}, which demonstrates the poor generalization and excess risk in DFL. Actually, it doesn't save as much bandwidth as one might think in practice, especially with high heterogeneity. When $m$ is large enough, even though CFL and DFL maintain similar communication costs, the generalization performance of CFL is much higher. This gives us some insights into the choice of modes. When the client scale is small, the advantages of decentralization are very prominent. It can easily achieve the same accuracy as centralized training while requiring fewer communication bits. In the large scale training, decentralized performance is significantly affected, ultimately leading to larger test errors.

\begin{table}[t]
  \caption{Performance collapse in DFL if the topology does not satisfy the minimal condition. We fix $S=100$~(upper) / $S=1000$~(lower) and increase the $m$ to test performance trends~(arrow) on common topologies.}
  \label{DFL_m}
  \centering
  \small
  \begin{tabular}{cccccccc}
    \toprule
    total $m$     &   300   &  350 &  400 &  450 & 500 & Trends \\
    \midrule
    {\ttfamily full}  & 67.86 & 69.90 & 70.94 & 72.10 & \textbf{73.12} & $\uparrow$ \\
    {\ttfamily exp}   & 64.64 & 65.85 & 66.23 & 67.04 & \textbf{67.28} & $\uparrow$ \\
    \midrule
    {\ttfamily grid}  & 57.28 & 58.14 & \textbf{59.11} & 58.31 & 58.20 & $\downarrow$ \\
    {\ttfamily ring}  & 43.68 & \textbf{44.54} & 43.56 & 42.01 & 41.19 & $\downarrow$ \\
    \midrule
    \midrule
    total $m$     &  30 &  35 &  40 & 45 & 50 \\
    \midrule
    {\ttfamily full}  & 71.36 & 72.83 & 73.72 & 74.90 & \textbf{75.69} & $\uparrow$ \\
    {\ttfamily exp}   & 71.10 & 72.17 & 72.56 & 73.98 & \textbf{75.13} & $\uparrow$ \\
    \midrule
    {\ttfamily grid}  & 66.82 & \textbf{68.17} & 67.85 & 66.94 & 66.32 & $\downarrow$ \\
    {\ttfamily ring}  & \textbf{63.50} & 63.34 & 63.11 & 61.08 & 58.30 & $\downarrow$ \\
    \bottomrule
  \end{tabular}
\end{table}

\quad 

\textbf{Performance Collapse.} Another important point of our analysis is the potential performance collapse in DFL. In Table~\ref{acc}, because the total amount of data remains unchanged, the number of local data samples $S$ will decrease as $m$ increases. To eliminate this impact, we fix the local amount of data $S=100 / 1000$ for each client in the horizontal comparison to validate the minimal condition required to avoid performance collapse. $m\times S$ means the total data samples. Under a fixed $S$, increasing $m$ also means enlarging the dataset. As shown in Table~\ref{DFL_m}, we can clearly see that on the {\ttfamily full} and {\ttfamily exp} topologies, increasing clients can effectively increase the test accuracy. However, on the {\ttfamily grid} and {\ttfamily ring} topologies, since they do not satisfy the minimal condition of $\kappa_\lambda$ as shown in Table~\ref{DFL_kappa}, even increasing $m$~(enlarging dataset) will cause significant performance degradation. It can also be seen that increasing $S$ is more efficient than increasing $m$ under the total sample size of data $Sm$, which is also consistent with our theoretical analysis in terms of $S$ and $m$.

%% file: texts/conclusion.tex
\section{Conclusion}
\label{conclusion}

In this paper, we provide the analysis of the uniform stability and excess risk between CFL and DFL without the idealized assumption of bounded gradients. Our analysis provides a comprehensive and novel understanding of the comparison between CFL and DFL. From the generalization perspective, CFL is always better than DFL. Furthermore, we prove that to achieve minimal excess risk and test error, CFL only requires partial local clients to participate in the training per round. Moreover, though decentralized approaches are adopted as the compromise of centralized ones which could significantly reduce the communication rounds theoretically, the topology must satisfy the minimal requirement to avoid performance collapse. In summary, our analysis clearly answers the \textit{question} in Sec.Introduction and points out how to choose the suitable training mode in real-world scenarios.

\section{Proofs of key lemmas}
In this part, we mainly introduce the proofs of the some important lemmas which are adopted in this work. Both centralized and decentralized FL setups minimize the finite-sum problem. Therefore, we denote $\Delta_k^t=\sum_{i\in\left[m\right]}\Vert w_{i,k}^t - \widetilde{w}_{i,k}^t \Vert$ as the average difference. Here we define an event $\xi$. If $\Delta_{k_0}^{t_0}=0$ happens at $k_0$-th iteration of $t_0$-th round, $\xi=1$ otherwise $0$. Because the index of the data samples on the joint dataset $\mathcal{C}$ and $\widetilde{\mathcal{C}}$ are selected simultaneously, it describes whether the different samples in the two datasets have been selected before $k_0$-th iteration of $t_0$-th round. We denote $\tau=tK+k$ as the index of training iterations. We denote $z_{i^\star,j^\star}$ and $\widetilde{z}_{i^\star,j^\star}$ are the only different data samples between the dataset $\mathcal{C}$ and $\widetilde{\mathcal{C}}$. The following Table~\ref{ap:notation} summarizes the details.
\begin{table}[H]
  \caption{Additional notations adopted in the proofs.}
  \label{ap:notation}
  \centering
  \begin{tabular}{ccc}
    \toprule
    Notation     &   Formulation   & Description \\
    \midrule
    $\tau$ & $tK+k$ & iterations\\
    $\tau_0$ & $t_0K+k_0$ & observed iteration\\
    $\left(i^\star, j^\star\right)$ & - & different data sample \\
    $\Delta_k^t$   & $\sum_{i\in\left[m\right]}\Vert w_{i,k}^t - \widetilde{w}_{i,k}^t \Vert$ & local stability\\
    \bottomrule
  \end{tabular}
\end{table}

\begin{lemma}[Stability in {\ttfamily FedAvg}]
\label{lemma:stability_bound}
    Let function $f(w)$ satisfies Assumption~\ref{assumption:lipschitz}, the models $w^T=\mathcal{A}(\mathcal{C})$ and $\widetilde{w}^T=\mathcal{A}(\widetilde{\mathcal{C}})$ are generated after $T$ training rounds by the centralized {\ttfamily FedAvg} method~(Algorithm~\ref{algorithm:fedavg}), we can bound their objective difference:
    \begin{equation}
        \mathbb{E}\left[\vert f(w^T;z) - f(\widetilde{w}^T;z)\vert\right] \leq G\mathbb{E}\left[\Vert w^T - \widetilde{w}^T\Vert \ \vert \ \xi \right] + \frac{nU\tau_0}{mS}.
    \end{equation}
    where $U=\sup_{w,z}f(w;z)<+\infty$ is the upper bound of the loss and $\tau_0=t_0K+k_0$ is a specific index of the iterations. 
\end{lemma}
\begin{proof}
    Via the expansion of the probability we have:
\begin{align*}
    &\quad \ \ \mathbb{E}\left[\vert f(w^T;z) - f(\widetilde{w}^T;z)\vert\right]\\
    &= P(\xi)\mathbb{E}\left[\vert f(w^T;z) - f(\widetilde{w}^T;z)\vert \ \vert \ \xi \right] \\
    &\quad + P(\xi^c)\mathbb{E}\left[\vert f(w^T;z) - f(\widetilde{w}^T;z)\vert \ \vert \ \xi^c \right]\\
    &\leq G\mathbb{E}\left[\Vert w^T - \widetilde{w}^T\Vert \ \vert \ \xi \right] + UP(\xi^c).
\end{align*}
Let the variable $I$ assume the index of the first time to use the data sample $\widetilde{z}_{i^\star,j^\star}$ on the dataset $\widetilde{\mathcal{S}}_{i^\star}$. When $I>t_0K+k_0$, then $\Delta_{k_0}^{t_0}=0$ must happens. Thus we have:
\begin{equation}
    P(\xi^c)=P(\Delta_{k_0}^{t_0}>0)\leq P(I\leq t_0K+k_0).
\end{equation}
On each step $\tau$, the data sample is uniformly sampled from the local dataset. When the dataset $\mathcal{S}_{i^\star}$ is selected, the probability of sampling $z_{i^\star,j^\star}$ is $1/S$. Let $\chi$ denote the event that $\mathcal{S}_{i^\star}$ is selected or not. Thus we have the union bound:
\begin{align*}
    &\quad \ P(I\leq t_0K+k_0)\\ 
    &\leq \sum_{t=0}^{t_0-1}\sum_{k=0}^{K-1} P(I = tK+k;\chi) + \sum_{k=0}^{k_0}P(I = t_0K+k;\chi)\\
    &= \sum_{t=0}^{t_0-1}\sum_{k=0}^{K-1}\sum_{\chi} P(I = tK+k\vert\chi)P(\chi)\\
    &\quad + \sum_{k=0}^{k_0}\sum_{\chi} P(I = t_0K+k\vert\chi)P(\chi)\\
    &= \frac{n}{m}\left(\sum_{t=0}^{t_0-1}\sum_{k=0}^{K-1} P(I = tK+k) + \sum_{k=0}^{k_0}P(I = t_0K+k)\right)\\
    &= \frac{n(t_0K+k_0)}{mS} = \frac{n\tau_0}{mS}.
\end{align*}
The second equality adopts the fact of random active clients with the probability of $n/m$.
\end{proof}

\begin{lemma}[Stability in {\ttfamily D-FedAvg}]
\label{lemma:dfl_stability_bound}
    Let function $f(w)$ satisfies Assumption~\ref{assumption:lipschitz}, the models $w^T=\mathcal{A}(\mathcal{C})$ and $\widetilde{w}^T=\mathcal{A}(\widetilde{\mathcal{C}})$ are generated after $T$ training rounds by the centralized {\ttfamily D-FedAvg} method~(Algorithm~\ref{algorithm:D-FedAvg}), we can bound their objective difference as:
    \begin{equation}
        \mathbb{E}\left[\vert f(w^T;z) - f(\widetilde{w}^T;z)\vert\right] \leq G\mathbb{E}\left[\Vert w^T - \widetilde{w}^T\Vert \ \vert \ \xi \right] + \frac{U\tau_0}{S}.
    \end{equation}
\end{lemma}
\begin{proof}
    The most part is the same as the proof in Lemma~\ref{lemma:stability_bound} except the probability $P(\chi)=1$ in a decentralized federated learning setup~(because all clients participate in). 
\end{proof}

\begin{lemma}[Topology-aware coefficient $\kappa_\lambda$]
\label{lambda}
    For $0 < \lambda < 1$ and $0 < \alpha < 1$, we have the following inequality:
    \begin{equation}
        \sum_{s=0}^{t-1}\frac{\lambda^{t-s-1}}{\left(s+1\right)^\alpha}\leq \frac{\kappa_\lambda}{t^\alpha}, 
    \end{equation}
    where $\kappa_\lambda=\left(\frac{\alpha}{e}\right)^\alpha\frac{1}{\lambda\left(\ln\frac{1}{\lambda}\right)^\alpha} + \frac{2^\alpha}{\left(1-\alpha\right)e\lambda\ln\frac{1}{\lambda}} + \frac{2^\alpha}{\lambda\ln\frac{1}{\lambda}}$.
\end{lemma}
\begin{proof}
    According to the accumulation, we have:
\begin{align*}
    &\quad \ \sum_{s=0}^{t-1}\frac{\lambda^{t-s-1}}{\left(s+1\right)^\alpha}\\
    &= \lambda^{t-1} + \sum_{s=1}^{t-1}\frac{\lambda^{t-s-1}}{\left(s+1\right)^\alpha}\leq \lambda^{t-1} + \int_{s=1}^{s=t}\frac{\lambda^{t-s-1}}{s^\alpha}ds\\
    &= \lambda^{t-1} + \int_{s=1}^{s=\frac{t}{2}}\frac{\lambda^{t-s-1}}{s^\alpha}ds + \int_{s=\frac{t}{2}}^{s=t}\frac{\lambda^{t-s-1}}{s^\alpha}ds\\
    &\leq \lambda^{t-1} + \lambda^{\frac{t}{2}-1}\int_{s=1}^{s=\frac{t}{2}}\frac{1}{s^\alpha}ds + \left(\frac{2}{t}\right)^\alpha\int_{s=\frac{t}{2}}^{s=t}\lambda^{t-s-1}ds\\
    &\leq \lambda^{t-1} + \lambda^{\frac{t}{2}-1}\frac{1}{1-\alpha}\left(\frac{t}{2}\right)^{1-\alpha} + \left(\frac{2}{t}\right)^\alpha\frac{\lambda^{-1}}{\ln\frac{1}{\lambda}}.
\end{align*}
Thus we have $\text{LHS}\leq\frac{1}{t^\alpha}\left(\lambda^{t-1}t^\alpha+\lambda^{\frac{t}{2}-1}\frac{t}{\left(1-\alpha\right)2^{1-\alpha}} + \frac{2^\alpha}{\lambda\ln\frac{1}{\lambda}}\right)$. The first term can be bounded as $\lambda^{t-1}t^\alpha\leq\left(\frac{\alpha}{e}\right)^\alpha\frac{1}{\lambda\left(\ln\frac{1}{\lambda}\right)^\alpha}$ and the second term can be bounded as $\lambda^{\frac{t}{2}-1}t\leq\frac{2}{e\lambda\ln\frac{1}{\lambda}}$, which indicates the selection of the constant $\kappa_\lambda=\left(\frac{\alpha}{e}\right)^\alpha\frac{1}{\lambda\left(\ln\frac{1}{\lambda}\right)^\alpha} + \frac{2^\alpha}{\left(1-\alpha\right)e\lambda\ln\frac{1}{\lambda}} + \frac{2^\alpha}{\lambda\ln\frac{1}{\lambda}}$. Furthermore, if $0<\alpha\leq \frac{1}{2}< 1$, we have $\kappa_\lambda\leq\frac{1}{\lambda\left(\ln\frac{1}{\lambda}\right)^\alpha}+\frac{2\sqrt{2}}{e\lambda\ln\frac{1}{\lambda}}+\frac{\sqrt{2}}{\lambda\ln\frac{1}{\lambda}}\leq\max\left\{\frac{1}{\lambda},\frac{1}{\lambda\sqrt{\ln\frac{1}{\lambda}}}\right\}+\frac{\left(2+e\right)\sqrt{2}}{e\lambda\ln\frac{1}{\lambda}}=\mathcal{O}\left(\max\left\{\frac{1}{\lambda},\frac{1}{\lambda\sqrt{\ln\frac{1}{\lambda}}}\right\}+\frac{1}{\lambda\ln\frac{1}{\lambda}}\right)$ w.r.t. $\lambda$.
\end{proof}

\begin{lemma}[Same Sample]
\label{same_data_update}
    Let the function $f_i$ satisfies Assumption~\ref{assumption:smooth}, and the local updates be $w_{i,k+1}^t=w_{i,k}^t-\eta g_{i,k}^t$ and $\widetilde{w}_{i,k+1}^t=\widetilde{w}_{i,k}^t-\eta\widetilde{g}_{i,k}^t$, by sampling the same data $z$~(not the $z_{i^\star,j^\star}$), we have:
    \begin{equation}
        \mathbb{E}\Vert w_{i,k+1}^t - \widetilde{w}_{i,k+1}^t\Vert\leq(1+\eta L)\mathbb{E}\Vert w_{i,k}^t - \widetilde{w}_{i,k}^t\Vert.
    \end{equation}
\end{lemma}
\begin{proof}
    In each round $t$, we have:
\begin{align*}
    &\quad \ \mathbb{E}\Vert w_{i,k+1}^t - \widetilde{w}_{i,k+1}^t\Vert\\
    &= \mathbb{E}\Vert w_{i,k}^t - \widetilde{w}_{i,k}^t - \eta(g_{i,k}^t - \widetilde{g}_{i,k}^t)\Vert\\
    &\leq \mathbb{E}\Vert w_{i,k}^t - \widetilde{w}_{i,k}^t\Vert + \eta\mathbb{E}\Vert \nabla f_{i}(w_{i,k}^t,z) - \nabla f_{i}(\widetilde{w}_{i,k}^t,z)\Vert\\
    &\leq (1+\eta L)\mathbb{E}\Vert w_{i,k}^t - \widetilde{w}_{i,k}^t\Vert.
\end{align*}
\end{proof}

\begin{lemma}[Different Sample]
\label{different_data_update}
    Let the function $f_i$ satisfies Assumption~\ref{assumption:smooth} and \ref{assumption:stochastic}, and the local updates be $w_{i^\star,k+1}^t=w_{i^\star,k}^t-\eta g_{i^\star,k}^t$ and $\widetilde{w}_{i^\star,k+1}^t=\widetilde{w}_{i^\star,k}^t-\eta\widetilde{g}_{i^\star,k}^t$, by sampling the different data samples $z_{i^\star,j^\star}$ and $\widetilde{z}_{i^\star,j^\star}$~(simplified to $z$ and $\widetilde{z}$), we have:
    \begin{equation}
        \mathbb{E}\Vert w_{i^\star,k+1}^t - \widetilde{w}_{i^\star,k+1}^t\Vert\leq (1+\eta L)\mathbb{E}\Vert w_{i^\star,k}^t - \widetilde{w}_{i^\star,k}^t\Vert + 2\eta\sigma_l.
    \end{equation}
\end{lemma}
\begin{proof}
    In each round $t$, we have:
\begin{align*}
    &\quad \ \mathbb{E}\Vert w_{i^\star,k+1}^t - \widetilde{w}_{i^\star,k+1}^t\Vert\\
    &= \mathbb{E}\Vert w_{i^\star,k}^t - \widetilde{w}_{i^\star,k}^t - \eta(g_{i^\star,k}^t - \widetilde{g}_{i^\star,k}^t)\Vert\\
    &\leq \mathbb{E}\Vert w_{i^\star,k}^t - \widetilde{w}_{i^\star,k}^t\Vert + \eta\mathbb{E}\Vert \nabla f_{i^\star}(w_{i^\star,k}^t,z) - \nabla f_{i^\star}(\widetilde{w}_{i^\star,k}^t,\widetilde{z})\Vert\\
    &\leq \mathbb{E}\Vert w_{i^\star,k}^t - \widetilde{w}_{i^\star,k}^t\Vert + \eta\mathbb{E}\Vert \nabla f_{i^\star}(w_{i^\star,k}^t,z) - \nabla f_{i^\star}(\widetilde{w}_{i^\star,k}^t,z)\Vert\\
    &\quad + \eta\mathbb{E}\Vert\nabla f_{i^\star}(\widetilde{w}_{i^\star,k}^t,z) - \nabla f_{i^\star}(\widetilde{w}_{i^\star,k}^t,\widetilde{z})\Vert\\
    &\leq (1+\eta L)\mathbb{E}\Vert w_{i^\star,k}^t - \widetilde{w}_{i^\star,k}^t\Vert \\
    &\quad + \eta\mathbb{E}\Vert\nabla f_{i^\star}(\widetilde{w}_{i^\star,k}^t,z) - \nabla f_{i^\star}(\widetilde{w}_{i^\star,k}^t)\Vert\\
    &\quad + \eta\mathbb{E}\Vert\nabla f_{i^\star}(\widetilde{w}_{i^\star,k}^t,\widetilde{z}) - \nabla f_{i^\star}(\widetilde{w}_{i^\star,k}^t)\Vert\\
    &\leq (1+\eta L)\mathbb{E}\Vert w_{i^\star,k}^t - \widetilde{w}_{i^\star,k}^t\Vert + 2\eta\sigma_l.
\end{align*}
The last inequality adopts the factor $\mathbb{E}\left[x\right]=\sqrt{\left(\mathbb{E}\left[x\right]\right)^2}=\sqrt{\mathbb{E}\left[x^2\right]-\mathbb{E}\left[x-\mathbb{E}\left[x\right]\right]^2}\leq\sqrt{\mathbb{E}\left[x^2\right]}$.
\end{proof}

%% file: texts/appendix.tex
\newpage
\onecolumn
\appendix

In the appendix, we state additional experiments~(Section~\ref{ap:additional_exp}) and introduce detailed proofs of the main theorems~(Section~\ref{ap:proof}). Our experimental results correspond to the optimal results of the each framework on each corresponding selection of the batchsize instead of a absolute comparison on the same batchsize. The optimal selection process are shown in Fig.~\ref{ap:fl_batchsize} and \ref{ap:dfl_batchsize}.

\section{Additional Setups and Experiments}
\label{ap:additional_exp}

\begin{figure}[H]
        \centering
        \subfloat[m=500, E=5]{
            \includegraphics[width=0.32\textwidth]{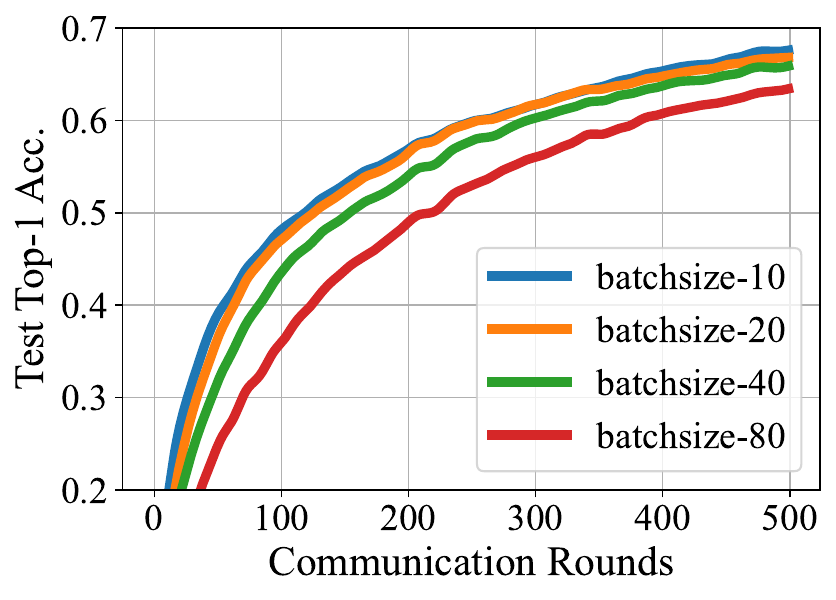}
            \label{ap:batchsizec1}
        }\!\!\!\!\!
        \subfloat[m=200, E=5]{
            \includegraphics[width=0.32\textwidth]{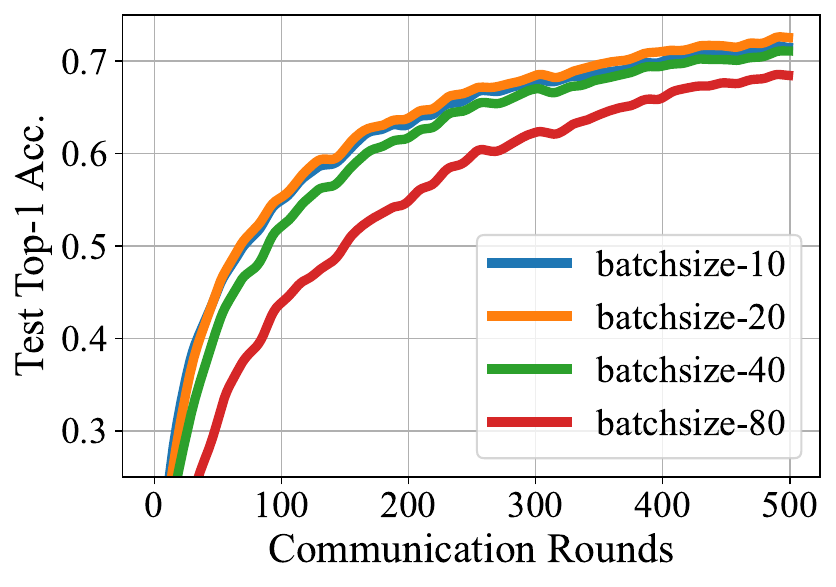}
            \label{ap:batchsizec2}
        }\!\!\!\!\!
        \subfloat[m=100, E=5]{
            \includegraphics[width=0.32\textwidth]{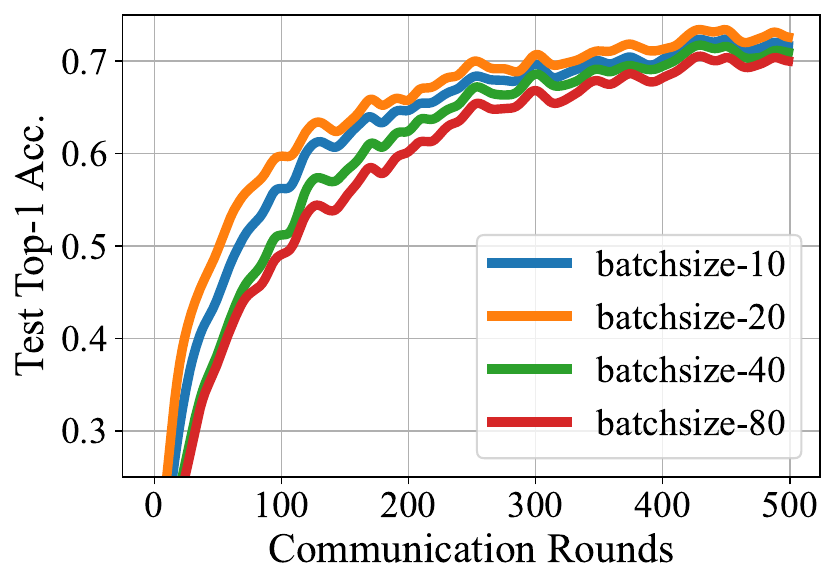}
            \label{ap:batchsizec3}
        }
        \quad
        \subfloat[m=500, E=20]{
            \includegraphics[width=0.32\textwidth]{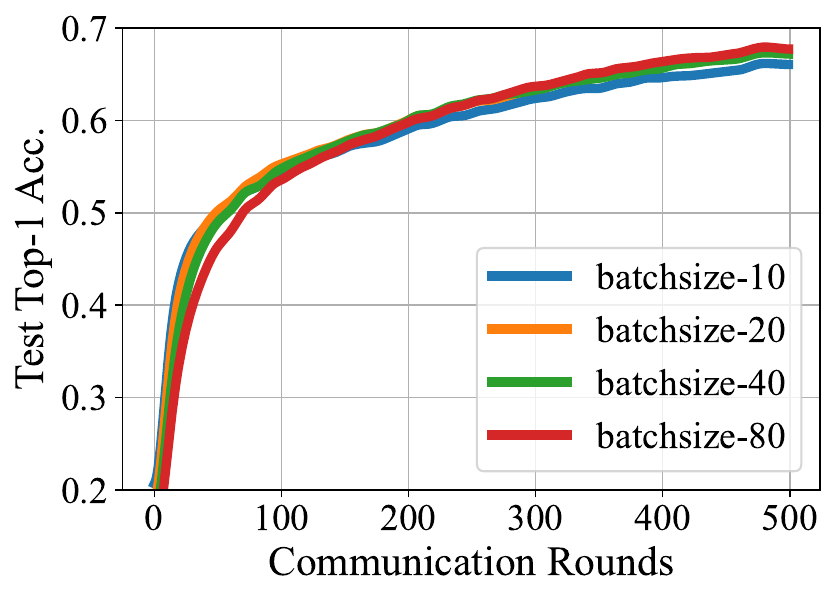}
            \label{ap:batchsizec4}
        }\!\!\!\!\!
        \subfloat[m=200, E=20]{
            \includegraphics[width=0.32\textwidth]{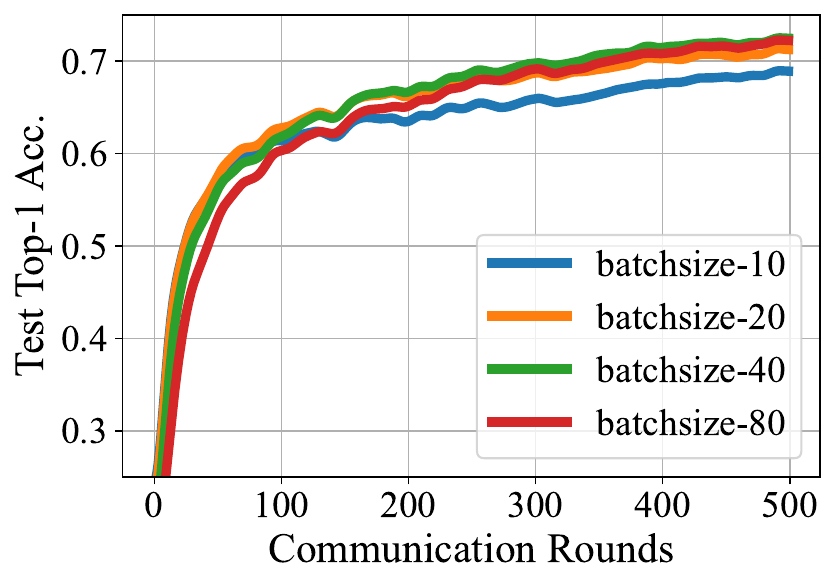}
            \label{ap:batchsizec5}
        }\!\!\!\!\!
        \subfloat[m=100, E=20]{
            \includegraphics[width=0.32\textwidth]{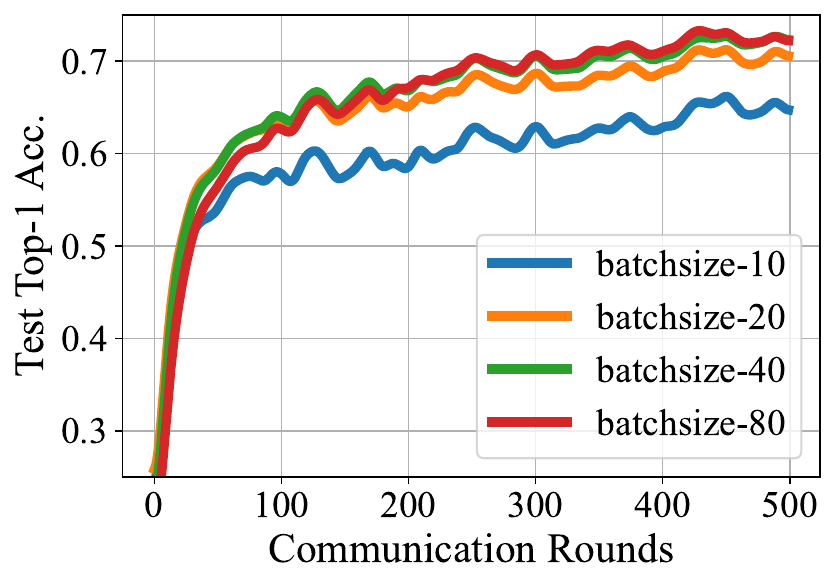}
            \label{ap:batchsizec6}
        }
        \caption{We test different batchsizes in \textbf{CFL} on the Dirichlet-0.1 split of the CIFAR-10 dataset with the ResNet-18 models. $m$ is the number of clients and $E$ is the number of local epochs.}
        \label{ap:fl_batchsize}
    \end{figure}
\begin{figure}[H]
        \centering
        \subfloat[m=500, E=5]{
            \includegraphics[width=0.32\textwidth]{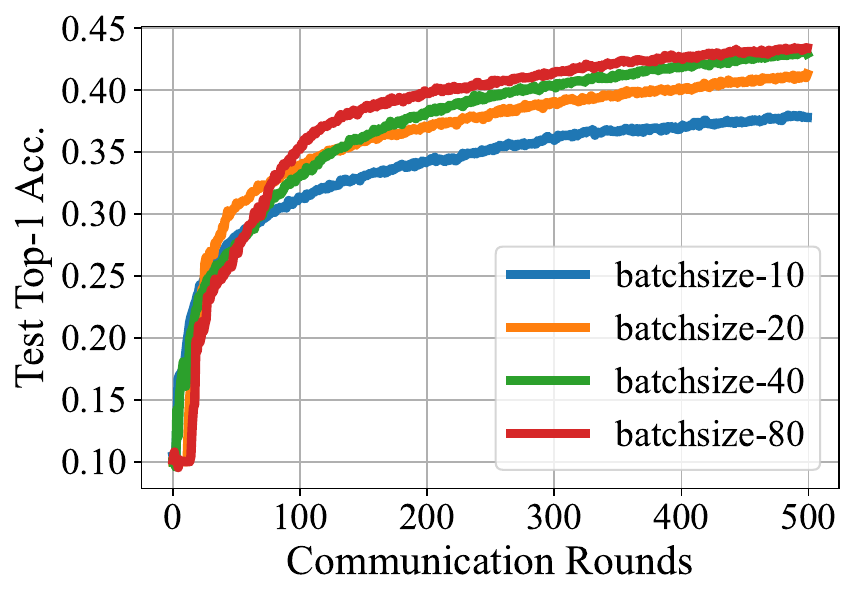}
            \label{ap:batchsized1}
        }\!\!\!\!\!
        \subfloat[m=200, E=5]{
            \includegraphics[width=0.32\textwidth]{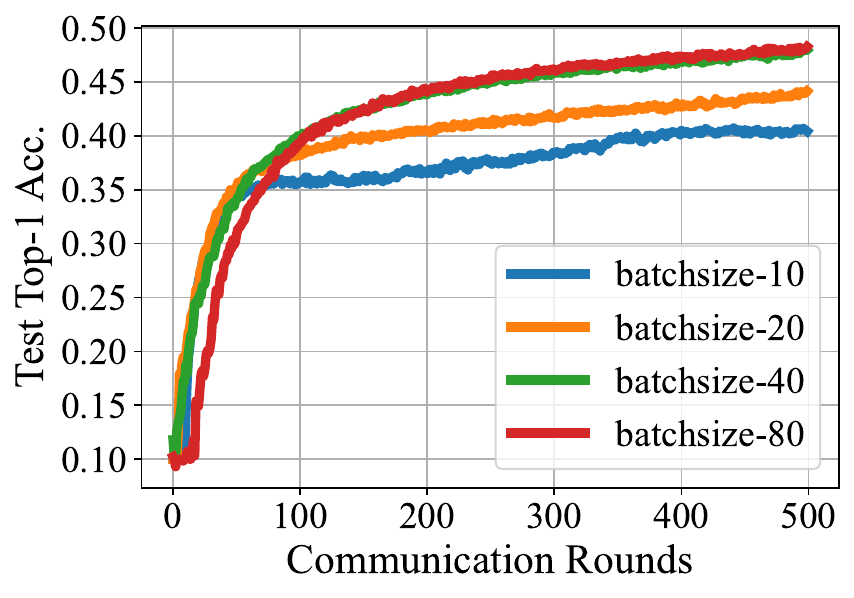}
            \label{ap:batchsized2}
        }\!\!\!\!\!
        \subfloat[m=100, E=5]{
            \includegraphics[width=0.32\textwidth]{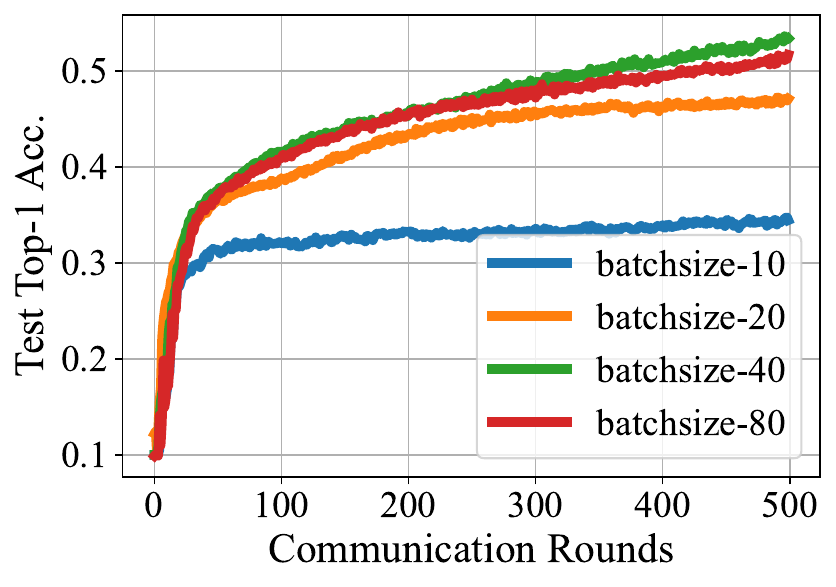}
            \label{ap:batchsized3}
        }
        \quad
        \subfloat[m=500, E=20]{
            \includegraphics[width=0.32\textwidth]{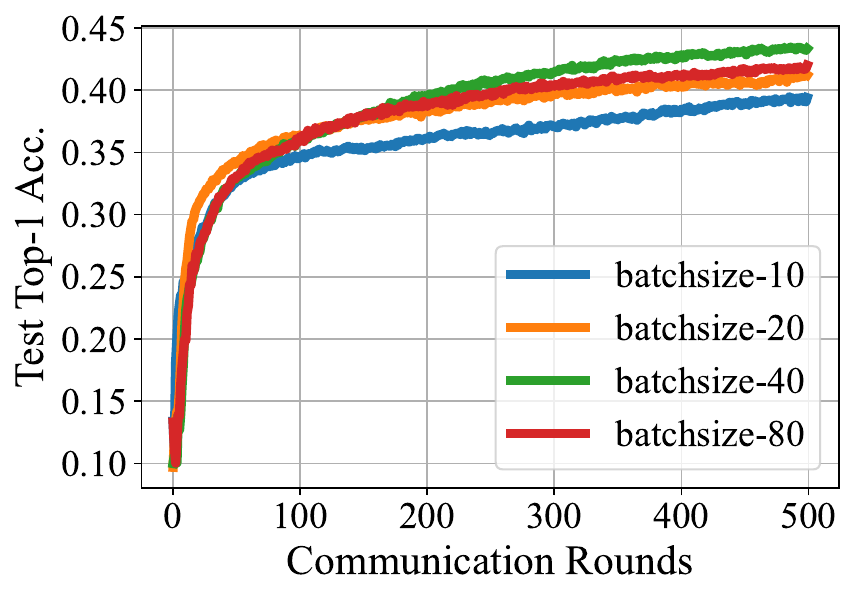}
            \label{ap:batchsized4}
        }\!\!\!\!\!
        \subfloat[m=200, E=20]{
            \includegraphics[width=0.32\textwidth]{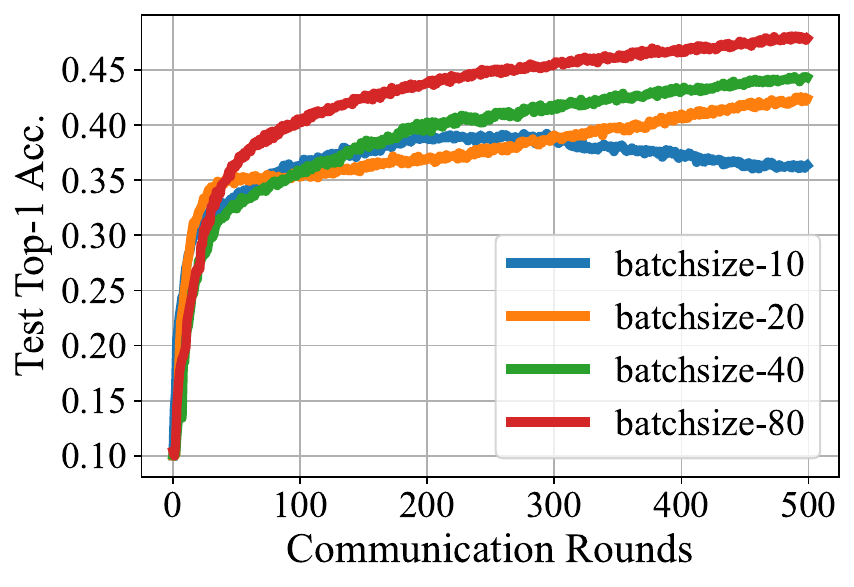}
            \label{ap:batchsized5}
        }\!\!\!\!\!
        \subfloat[m=100, E=20]{
            \includegraphics[width=0.32\textwidth]{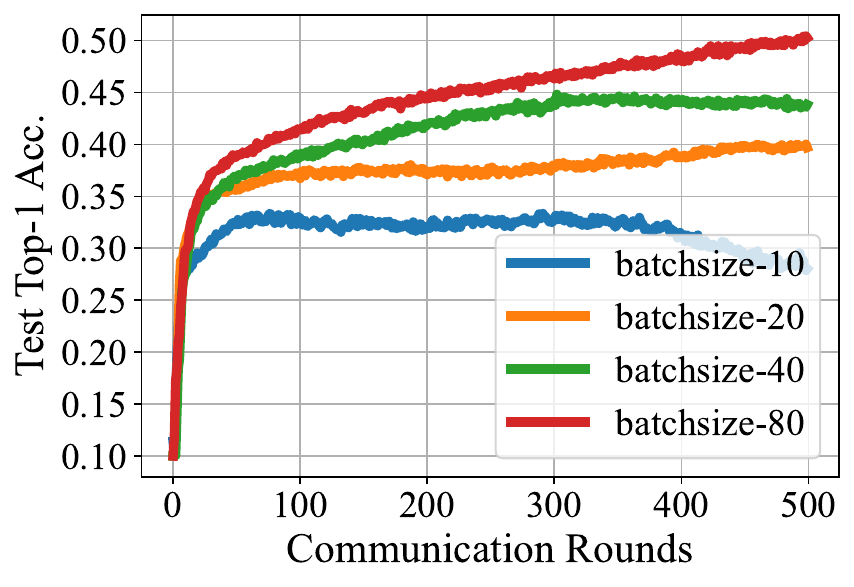}
            \label{ap:batchsized6}
        }
        \caption{We test different batchsizes in \textbf{DFL} on the Dirichlet-0.1 split of the CIFAR-10 dataset with the ResNet-18 models. $m$ is the number of clients and $E$ is the number of local epochs.}
        \label{ap:dfl_batchsize}
    \end{figure}

\subsection{Additional Experiments}
\subsubsection{Different Active Ratios in CFL}
\label{ap:fl_local_interval}
\begin{figure}[t]
        \centering
        \subfloat[m=500, E=5]{
            \includegraphics[width=0.32\textwidth]{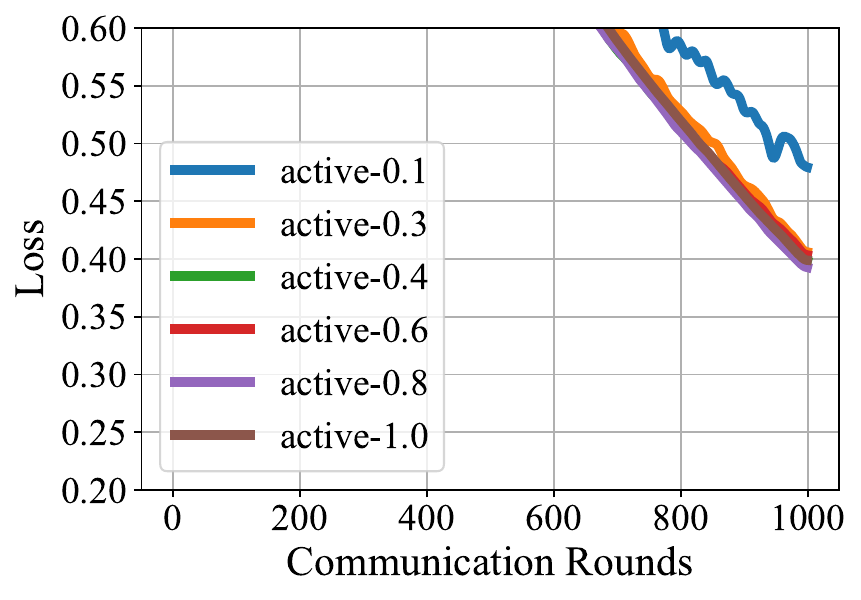}
            \label{ap:ratio1_loss}
        }\!\!\!\!\!
        \subfloat[m=200, E=5]{
            \includegraphics[width=0.32\textwidth]{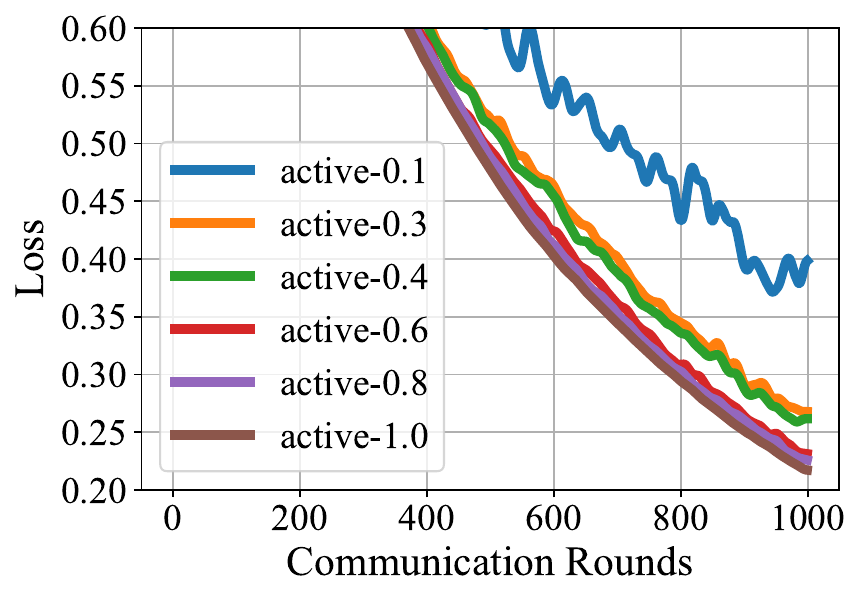}
            \label{ap:ratio2_loss}
        }\!\!\!\!\!
        \subfloat[m=100, E=5]{
            \includegraphics[width=0.32\textwidth]{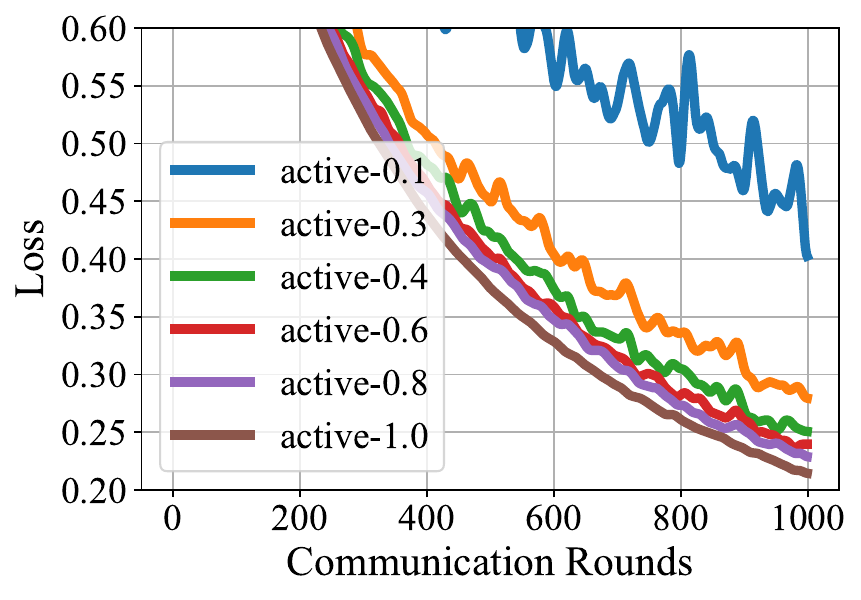}
            \label{ap:ratio3_loss}
        }
        \quad
        \subfloat[m=500, E=20]{
            \includegraphics[width=0.32\textwidth]{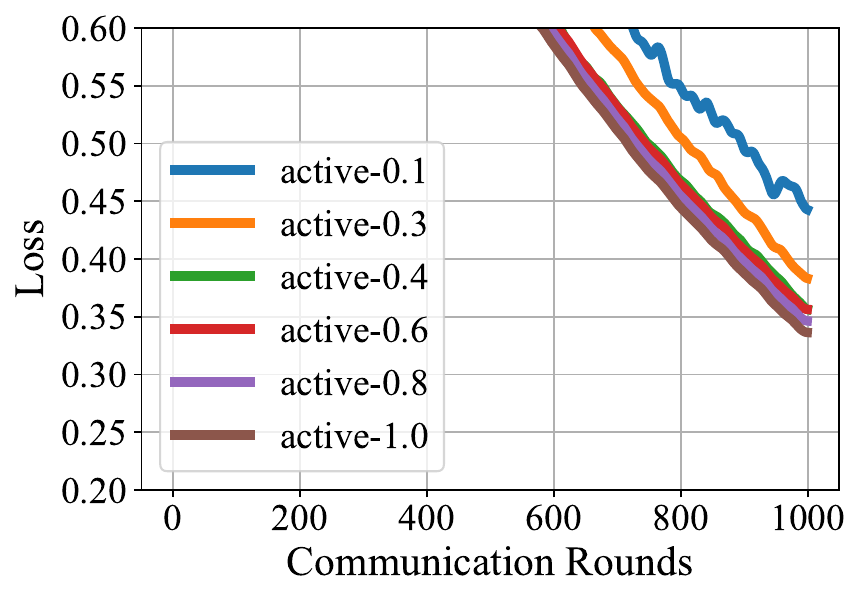}
            \label{ap:ratio4_loss}
        }\!\!\!\!\!
        \subfloat[m=200, E=20]{
            \includegraphics[width=0.32\textwidth]{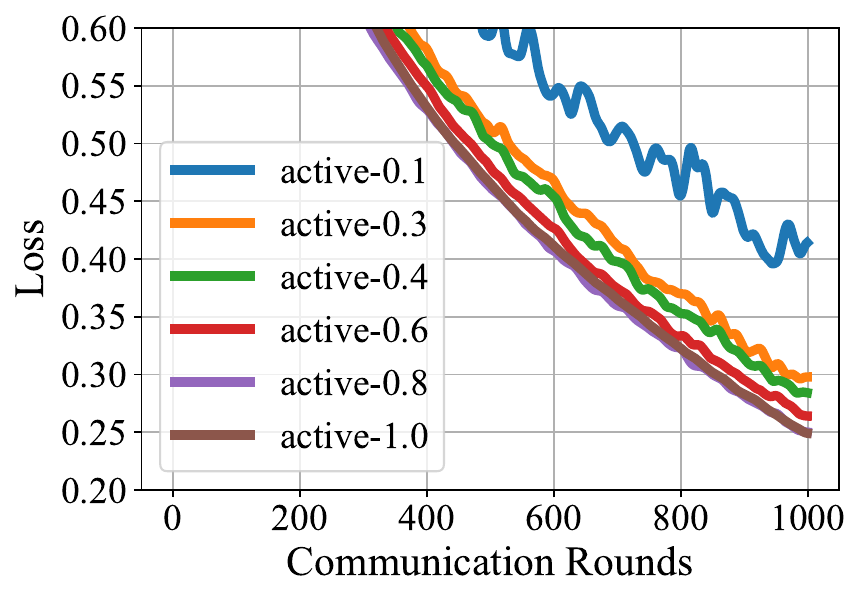}
            \label{ap:ratio5_loss}
        }\!\!\!\!\!
        \subfloat[m=100, E=20]{
            \includegraphics[width=0.32\textwidth]{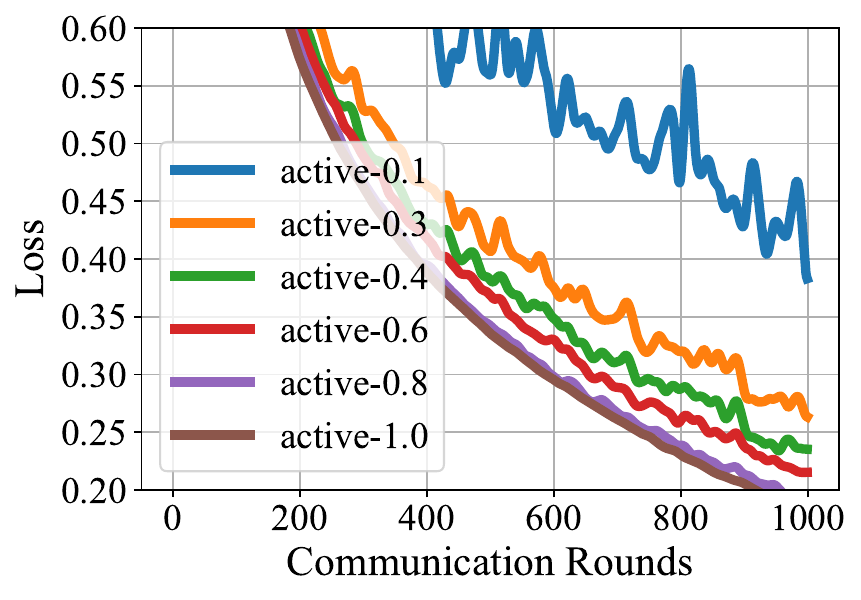}
            \label{ap:ratio6_loss}
        }
        \caption{\textbf{Loss} curves of different active ratios in CFL.}
        \label{ap:fl_active_loss}
        \vskip -0.4cm
    \end{figure}
Obviously, increasing $n$ helps to accelerate the optimization. A larger active ratio means a faster convergence rate. As shown in Figure~\ref{ap:fl_active_loss}, we can see this phenomenon very clearly in the subfloat~(c) and (f). Though the real acceleration is not as fast as linear speedup, from the optimization perspective, increasing the active ratio can truly achieve a lower loss value. This is also consistent with the conclusions of previous work in the optimization process analysis.
\begin{figure}[t]
        \centering
        \subfloat[m=500, E=5]{
            \includegraphics[width=0.32\textwidth]{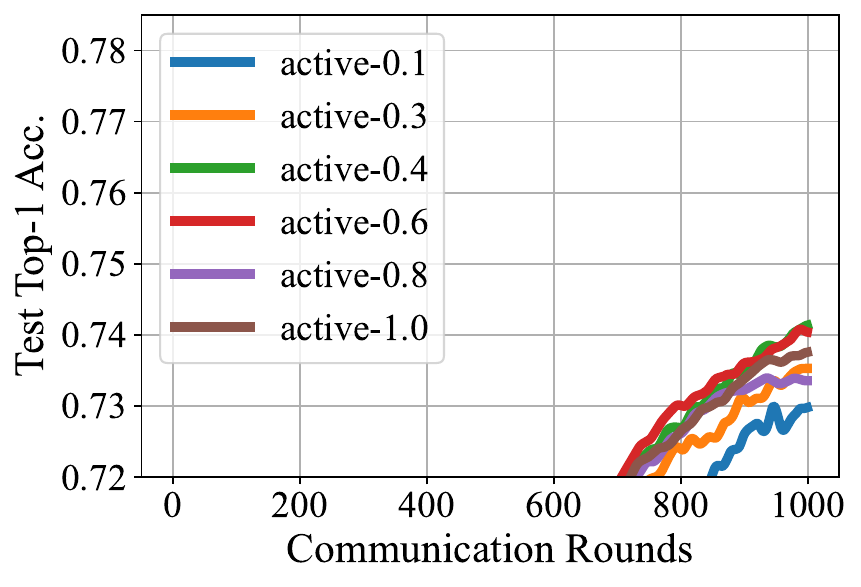}
            \label{ap:ratio1_acc}
        }\!\!\!\!\!
        \subfloat[m=200, E=5]{
            \includegraphics[width=0.32\textwidth]{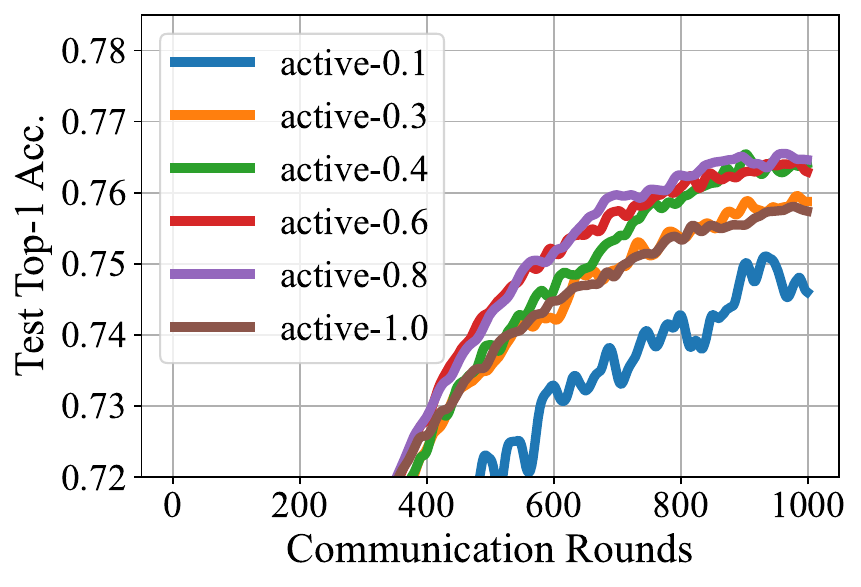}
            \label{ap:ratio2_acc}
        }\!\!\!\!\!
        \subfloat[m=100, E=5]{
            \includegraphics[width=0.32\textwidth]{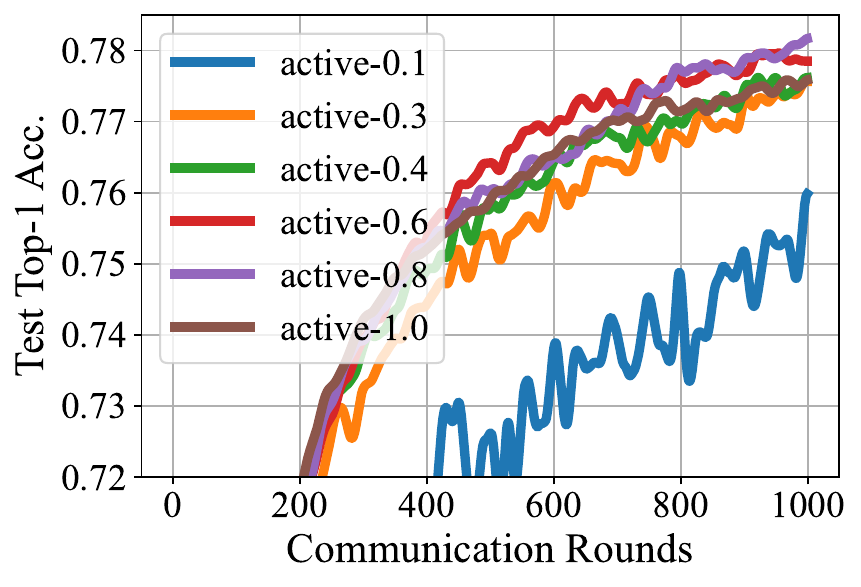}
            \label{ap:ratio3_acc}
        }
        \quad
        \subfloat[m=500, E=20]{
            \includegraphics[width=0.32\textwidth]{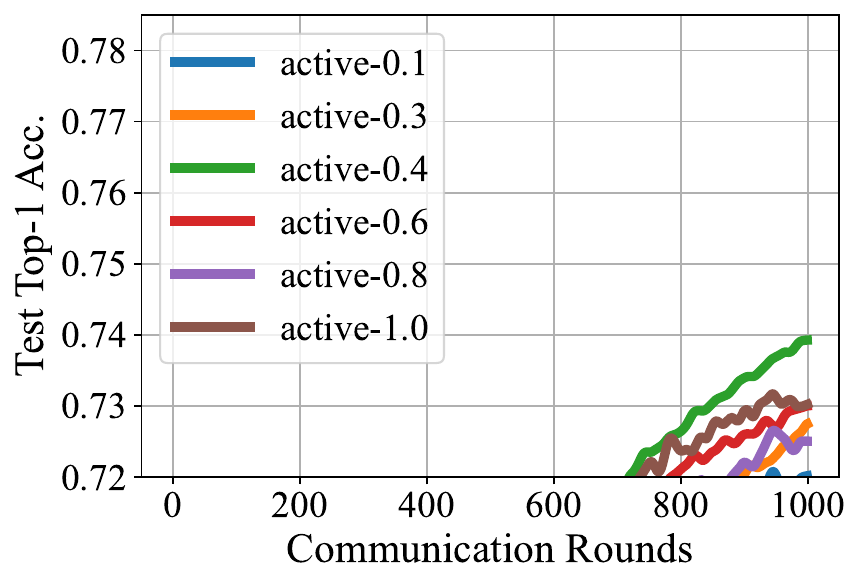}
            \label{ap:ratio4_acc}
        }\!\!\!\!\!
        \subfloat[m=200, E=20]{
            \includegraphics[width=0.32\textwidth]{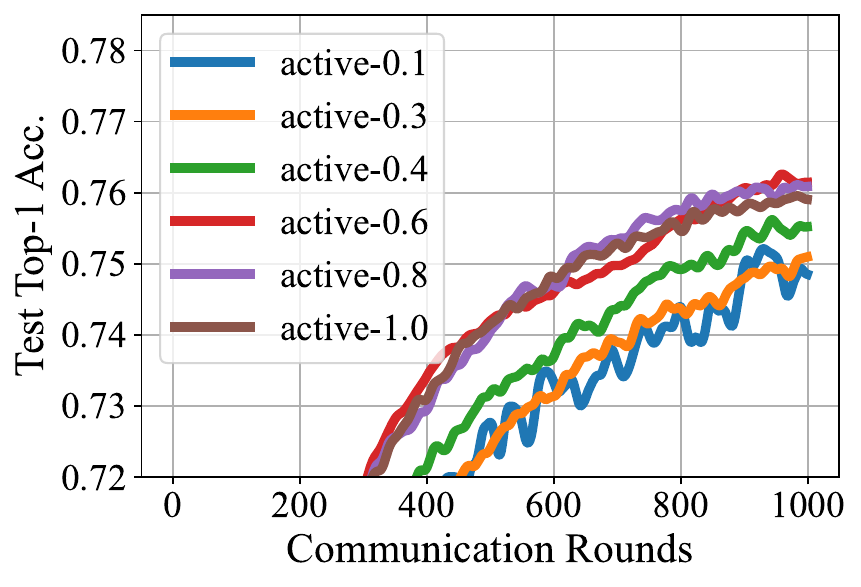}
            \label{ap:ratio5_acc}
        }\!\!\!\!\!
        \subfloat[m=100, E=20]{
            \includegraphics[width=0.32\textwidth]{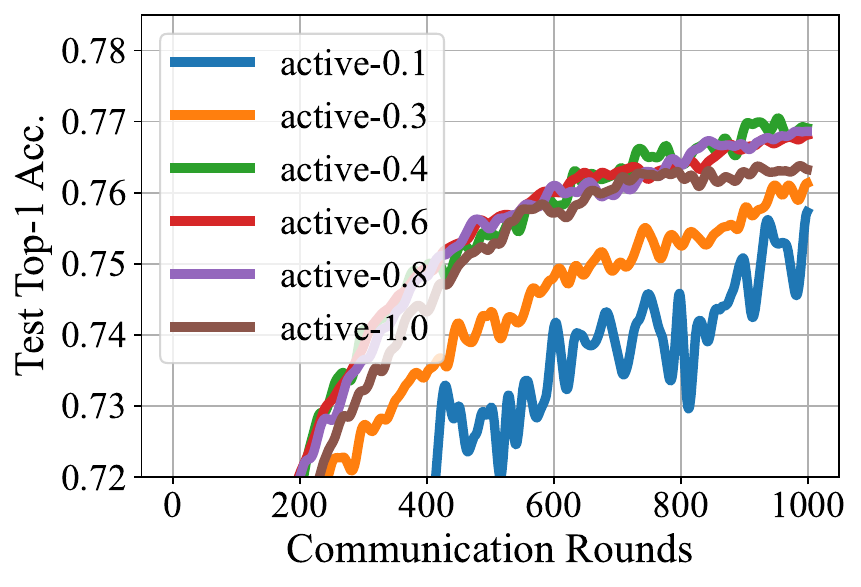}
            \label{ap:ratio6_acc}
        }
        \caption{\textbf{Accuracy} curves of different active ratios in CFL.}
        \label{ap:fl_active_acc}
        \vskip -0.4cm
    \end{figure}
However, as shown in Figure~\ref{ap:fl_active_acc}, increasing $n$ does not always mean higher test accuracy. In CFL, there is an optimal active ratio, which means the active number of clients is limited.

\subsubsection{Different Topology in DFL}
\label{ap:experiment2}
\begin{figure}[t]
        \centering
        \subfloat[m=500, E=5]{
            \includegraphics[width=0.32\textwidth]{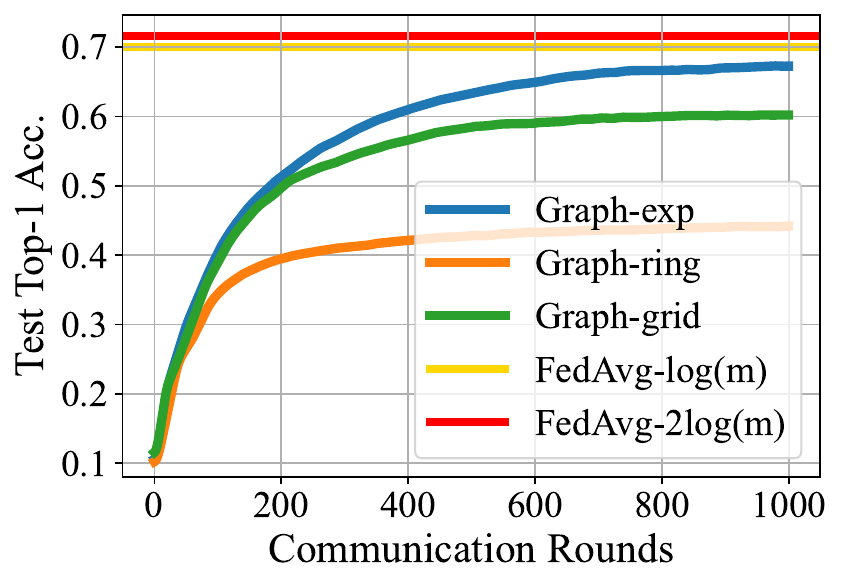}
            \label{ap:cfl_dfl1}
        }\!\!\!\!\!
        \subfloat[m=200, E=5]{
            \includegraphics[width=0.32\textwidth]{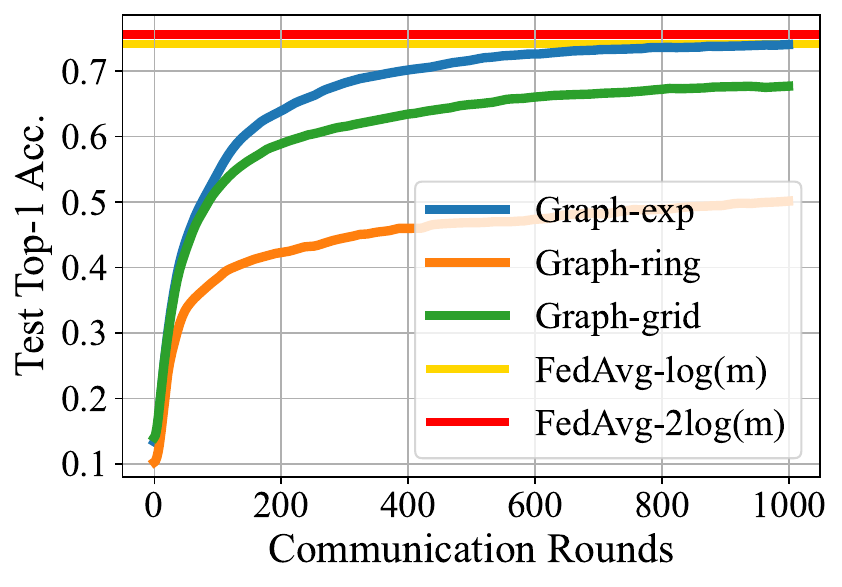}
            \label{ap:cfl_dfl2}
        }\!\!\!\!\!
        \subfloat[m=100, E=5]{
            \includegraphics[width=0.32\textwidth]{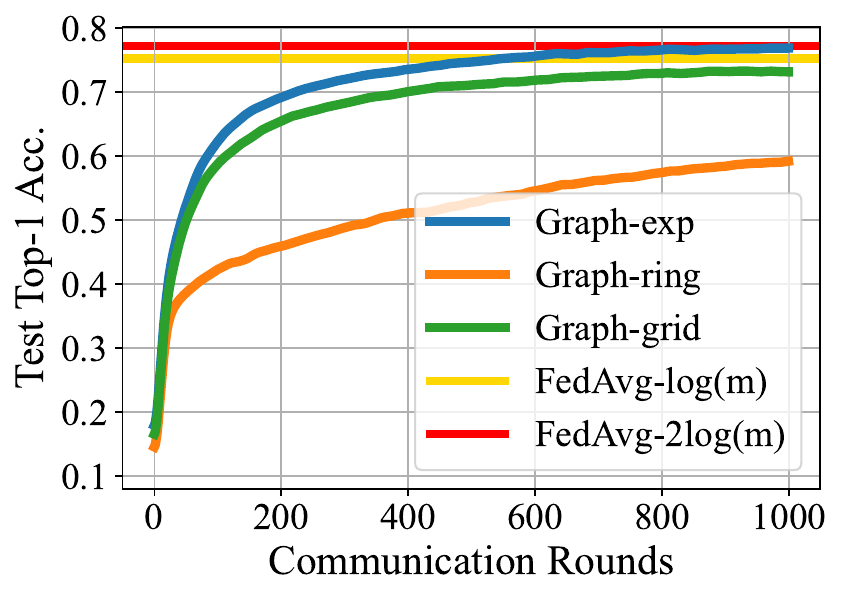}
            \label{ap:cfl_dfl3}
        }
        \quad
        \subfloat[m=500, E=20]{
            \includegraphics[width=0.32\textwidth]{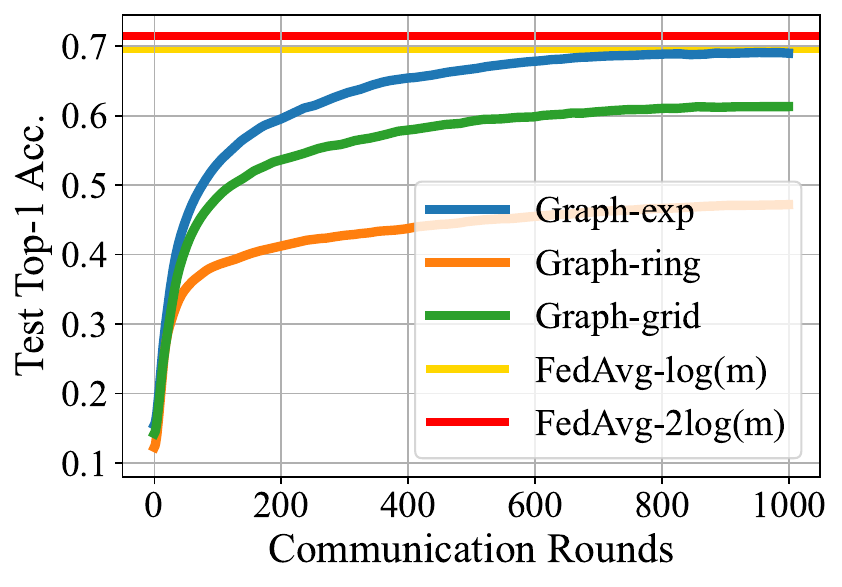}
            \label{ap:cfl_dfl4}
        }\!\!\!\!\!
        \subfloat[m=200, E=20]{
            \includegraphics[width=0.32\textwidth]{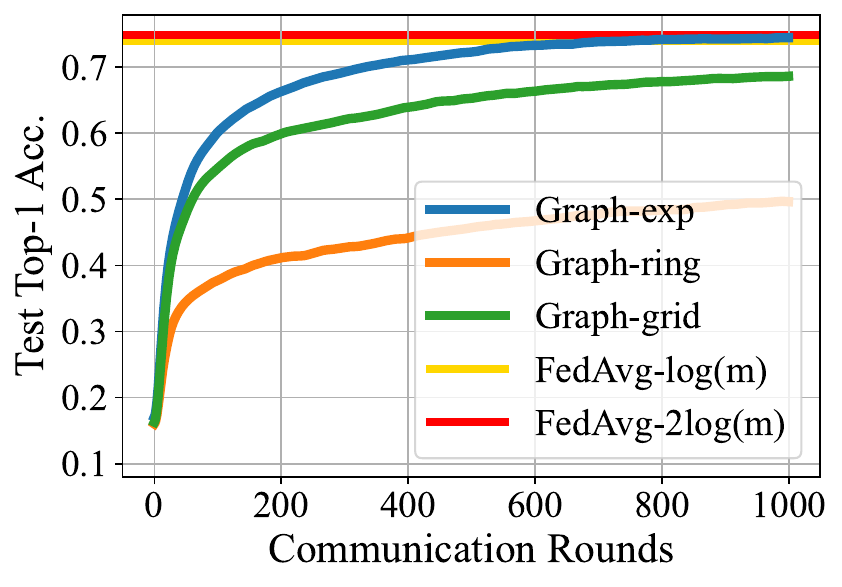}
            \label{ap:cfl_dfl5}
        }\!\!\!\!\!
        \subfloat[m=100, E=20]{
            \includegraphics[width=0.32\textwidth]{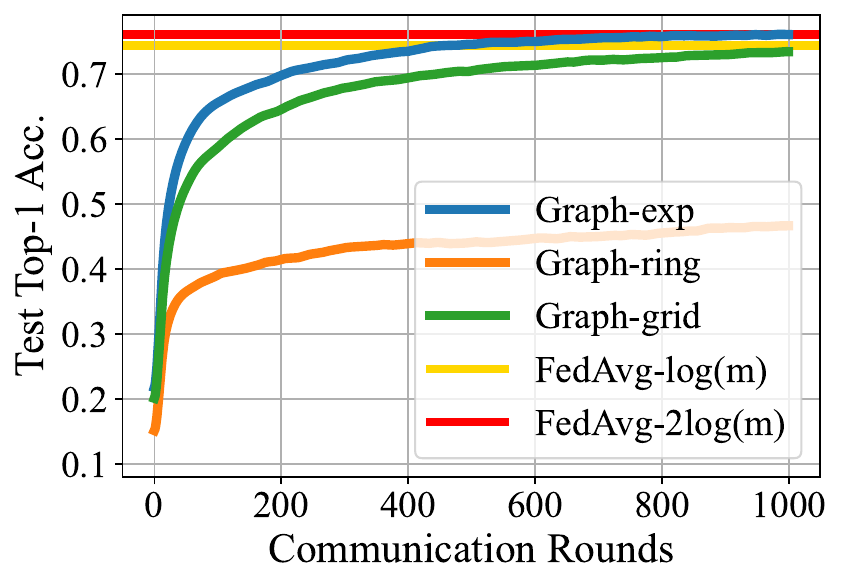}
            \label{ap:cfl_dfl6}
        }
        \caption{\textbf{Accuracy} curves of different topologies in DFL and corresponding CFL.}
        \label{ap:cfl_dfl}
        \vskip -0.4cm
    \end{figure}
As Figure~\ref{ap:cfl_dfl} shows, under similar communication costs, DFL can not be better than CFL. FL always maintains stronger generalization performance, except at very low active ratios where its performance is severely compromised. We can also observe a more significant phenomenon. The gap between CFL and DFL will be larger as $m$ increases.
\begin{figure}[t]
        \centering
        \subfloat[ring topology]{
            \includegraphics[width=0.32\textwidth]{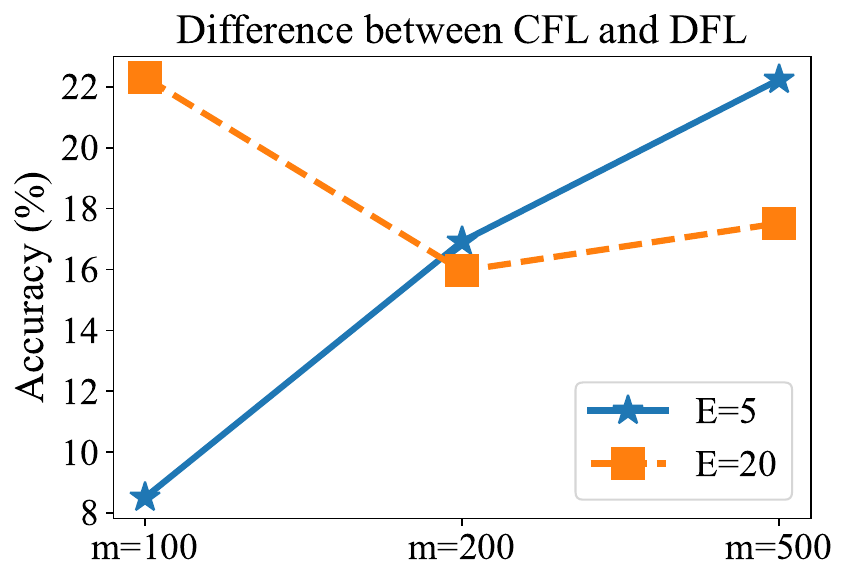}
            \label{ap:d_cfl_dfl1}
        }\!\!\!\!\!
        \subfloat[grid topology]{
            \includegraphics[width=0.32\textwidth]{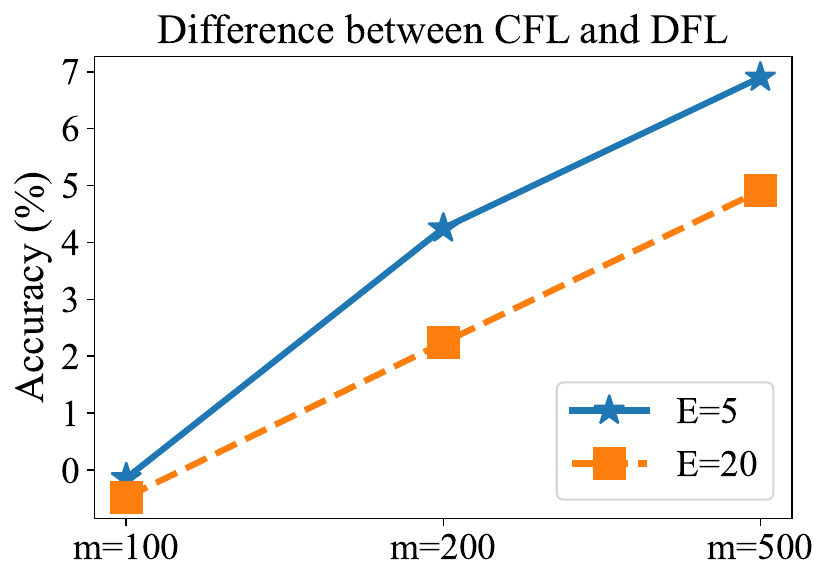}
            \label{ap:d_cfl_dfl2}
        }\!\!\!\!\!
        \subfloat[exp topology]{
            \includegraphics[width=0.32\textwidth]{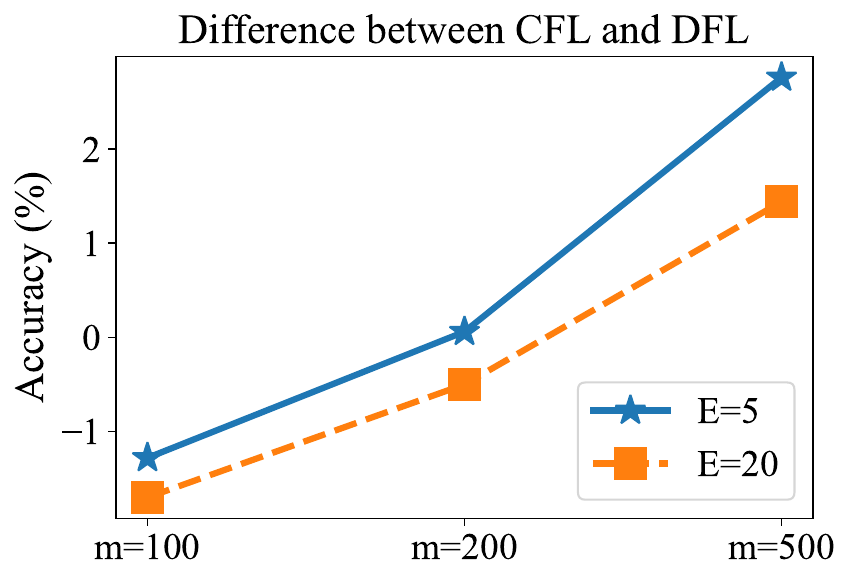}
            \label{ap:d_cfl_dfl3}
        }
        \caption{\textbf{Accuracy} difference of different topologies between DFL and corresponding CFL.}
        \label{ap:d_cfl_dfl}
        \vskip -0.4cm
    \end{figure}
As shown in Figure~\ref{ap:d_cfl_dfl}, we can clearly see that the difference between CFL and DFL increases significantly as $m$ increases. We calculate the difference by (\textit{test accuracy} of CFL - \textit{test accuracy} of DFL) under the same communication costs. According to our analysis, the excess risk of CFL is much smaller than that of DFL, which is $\mathcal{O}(m^{-\frac{1+\mu L}{1+2\mu L}})$ and $\mathcal{O}(m^{-\frac{1}{1+\mu L}})$ respectively. Therefore, we can know that CFL generalizes better than DFL at most with the rate of $\mathcal{O}(m^{\frac{\left(\mu L\right)^2}{\left(1+\mu L\right)\left(1+2\mu L\right)}})$. In general cases, the worst generalization of CFL equals to the best generalization of DFL. Therefore, when the generalization error dominates the test accuracy, CFL is always better than DFL.

\subsection{Proof of Theorems}
\label{ap:proof}

We first supplement two additional lemmas to facilitate the description of the proof for the conclusions.

\begin{lemma}[Upper Bound of Aggregation Gaps]
\label{aggregation}
    According to Algorithm~\ref{algorithm:fedavg} and \ref{algorithm:D-FedAvg}, the aggregation of centralized federated learning is $w_{i,0}^{t+1}=w^{t+1}=\frac{1}{n}\sum_{\mathcal{N}}w_{i,K}^{t}$, and the aggregation of decentralized federated learning is $w_{i,0}^{t+1}=\sum_{j\in\mathcal{A}_i}a_{ij}w_{i,K}^{t}$. On both setups, we can upper bound the aggregation gaps by:
    \begin{equation}
        \Delta_0^{t+1} \leq \Delta_K^{t}.
    \end{equation}
\end{lemma}
\begin{proof}
    We prove them respectively.\\\\
    (1) Centralized federated learning setup~\cite{acar2021federated}.
    
    In centralized federated learning, we select a subset $\mathcal{N}$ in each communication round. Thus we have:
    \begin{align*}
    &\quad \ \ \Delta_0^{t+1}\\
    &= \sum_{i\in[m]}\mathbb{E}\Vert w_{i,0}^{t+1} - \widetilde{w}_{i,0}^{t+1}\Vert = \sum_{i\in[m]}\mathbb{E}\Vert w^{t+1} - \widetilde{w}^{t+1}\Vert = \sum_{i\in[m]}\mathbb{E}\Vert \frac{1}{n}\sum_{i\in\mathcal{N}}\left(w_{i,K}^t - \widetilde{w}_{i,K}^t\right)\Vert\\
    &\leq \sum_{i\in[m]}\frac{1}{n}\mathbb{E}\left[\sum_{i\in\mathcal{N}}\Vert w_{i,K}^t - \widetilde{w}_{i,K}^t\Vert\right] = \sum_{i\in[m]}\frac{1}{n}\frac{n}{m}\sum_{i\in[m]}\mathbb{E}\Vert w_{i,K}^t - \widetilde{w}_{i,K}^t\Vert\\
    &= \sum_{i\in[m]}\frac{1}{m}\sum_{i\in[m]}\mathbb{E}\Vert w_{i,K}^t - \widetilde{w}_{i,K}^t\Vert = \sum_{i\in[m]}\mathbb{E}\Vert w_{i,K}^t - \widetilde{w}_{i,K}^t\Vert = \Delta_K^t.
\end{align*}
    (2) Decentralized federated learning setup.

    In decentralized federated learning, we aggregate the models in each neighborhood. Thus we have:
\begin{align*}
    &\quad \ \ \Delta_0^{t+1}\\
    &= \sum_{i\in[m]}\mathbb{E}\Vert w_{i,0}^{t+1} - \widetilde{w}_{i,0}^{t+1}\Vert = \sum_{i\in[m]}\mathbb{E}\Vert \sum_{j\in \mathcal{A}_i} a_{ij} \left(w_{j,K}^t - \widetilde{w}_{j,K}^t\right)\Vert\leq \sum_{i\in[m]}\sum_{j\in \mathcal{A}_i} a_{ij}\mathbb{E}\Vert  w_{j,K}^t - \widetilde{w}_{j,K}^t\Vert\\
    &= \sum_{j\in[m]}\sum_{i\in \mathcal{A}_j} a_{ji}\mathbb{E}\Vert  w_{j,K}^t - \widetilde{w}_{j,K}^t\Vert
    \leq \sum_{j\in[m]}\mathbb{E}\Vert w_{j,K}^{t} - \widetilde{w}_{j,K}^{t}\Vert = \Delta_K^{t}.
\end{align*}
The last equality adopts the symmetry of the adjacent matrix $\mathbf{A}=\mathbf{A}^\top$.
\end{proof}

\begin{lemma}[Recursion]
\label{recursion}
    According to the Lemma~\ref{same_data_update} and \ref{different_data_update}, we can bound the recursion in the local training:
    \begin{equation}
        \Delta_{k+1}^t+\frac{2\sigma_l}{SL}\leq\left(1+\eta L\right)\left(\Delta_{k}^t+\frac{2\sigma_l}{SL}\right).
    \end{equation}
\end{lemma}
\begin{proof}
    In each iteration, the specific ${j^\star}$-th data sample in the $\mathcal{S}_{i^\star}$ and $\widetilde{\mathcal{S}}_{i^\star}$ is uniformly selected with the probability of $1/S$. In other datasets~$\mathcal{S}_i$, all the data samples are the same. Thus we have:
\begin{align*}
    \Delta_{k+1}^t 
    &= \sum_{i\neq i^\star}\mathbb{E}\left[\Vert w_{i,k+1}^{t} - \widetilde{w}_{i,k+1}^{t} \Vert \right] + \mathbb{E}\left[\Vert w_{i^\star,k+1}^{t} - \widetilde{w}_{i^\star,k+1}^{t} \Vert\right]\\
    &\leq \left(1+\eta L\right)\sum_{i\neq i^\star}\mathbb{E}\left[\Vert w_{i,k+1}^{t} - \widetilde{w}_{i,k+1}^{t} \Vert\right] + \left(1-\frac{1}{S}\right)(1+\eta L)\mathbb{E}\left[\Vert w_{i^\star,k}^{t} - \widetilde{w}_{i^\star,k}^{t} \Vert\right]\\
    &\quad + \frac{1}{S}\left[\left(1+\eta L\right)\mathbb{E}\left[\Vert w_{i^\star,k}^{t} - \widetilde{w}_{i^\star,k}^{t} \Vert\right] + 2\eta\sigma_l\right] = \left(1+\eta L\right)\Delta_{k}^t + \frac{2\eta \sigma_l}{S}.
\end{align*}
    There we can bound the recursion formulation as $\Delta_{k+1}^t+\frac{2\sigma_l}{SL}\leq\left(1+\eta L\right)\left(\Delta_{k}^t+\frac{2\sigma_l}{SL}\right)$.
\end{proof}

\subsection{Stability of Centralized FL~(Theorem~\ref{thm1})}
According to the Lemma~\ref{aggregation} and \ref{recursion}, it is easy to bound the local stability term. We still obverse it when the event $\xi$ happens, and we have $\Delta_{k_0}^{t_0}=0$. Therefore, we unwind the recurrence formulation from $T,K$ to $t_0,k_0$. Let the learning rate $\eta=\frac{\mu}{\tau}=\frac{\mu}{tK+k}$ is decayed as the communication round $t$ and iteration $k$ where $\mu\leq\frac{1}{L}$ is a specific constant, we have:
\begin{align*}
    \Delta_K^T
    &\leq \left[\prod_{\tau=(T-1)K+1}^{TK}\left(1+\frac{\mu L}{\tau}\right)\right]\left(\Delta_{0}^T+\frac{2\sigma_l}{SL}\right)\leq \left[\prod_{\tau=(T-1)K+1}^{TK}\left(1+\frac{\mu L}{\tau}\right)\right]\left(\Delta_{K}^{T-1}+\frac{2\sigma_l}{SL}\right)\\
    &\leq \left[\prod_{\tau=t_0K+k_0+1}^{TK}\left(1+\frac{\mu L}{\tau}\right)\right]\left(\Delta_{k_0}^{t_0}+\frac{2\sigma_l}{SL}\right)\\
    &\leq \left[\prod_{\tau=t_0K+k_0+1}^{TK}e^{\left(\frac{\mu L}{\tau}\right)}\right]\left(\frac{2\sigma_l}{SL}\right)=e^{\mu L\left(\sum_{\tau=t_0K+k_0+1}^{TK}\frac{1}{\tau}\right)}\frac{2\sigma_l}{SL}\\
    &\leq e^{\mu L\ln\left(\frac{TK}{t_0K+k_0}\right)}\frac{2\sigma_l}{SL}\leq \left(\frac{TK}{\tau_0}\right)^{\mu L}\frac{2\sigma_l}{SL}.
\end{align*}
According to the Lemma~\ref{lemma:stability_bound}, the first term in the stability~(condition is omitted for abbreviation) can be bound as:
\begin{align*}
    &\quad \ \mathbb{E}\Vert w^{T+1} - \widetilde{w}^{T+1}\Vert
    = \mathbb{E}\Vert \frac{1}{n}\sum_{i\in\mathcal{N}}\left(w_{i,K}^T - \widetilde{w}_{i,K}^T\right)\Vert
    = \frac{1}{n}\mathbb{E}\Vert\sum_{i\in\mathcal{N}}\left(w_{i,K}^T - \widetilde{w}_{i,K}^T\right)\Vert\\
    &\leq \frac{1}{n}\mathbb{E}\sum_{i\in\mathcal{N}}\Vert\left(w_{i,K}^T - \widetilde{w}_{i,K}^T\right)\Vert
    =\frac{1}{n}\frac{n}{m}\mathbb{E}\sum_{i\in\left[m\right]}\Vert\left(w_{i,K}^T - \widetilde{w}_{i,K}^T\right)\Vert\\
    &= \frac{1}{m}\sum_{i\in\left[m\right]}\mathbb{E}\Vert\left(w_{i,K}^T - \widetilde{w}_{i,K}^T\right)\Vert = \frac{1}{m}\Delta_K^T \leq \left(\frac{TK}{\tau_0}\right)^{\mu L}\frac{2\sigma_l}{mSL}.
\end{align*}
Therefore, we can upper bound the stability in centralized federated learning as:
\begin{align*}
    \mathbb{E}\vert f(w^{T+1};z) - f(\widetilde{w}^{T+1};z)\vert
    &\leq G\mathbb{E}\Vert w^{T+1} - \widetilde{w}^{T+1}\Vert + \frac{nU\tau_0}{mS}
    \leq \frac{2\sigma_l G}{mSL}\left(\frac{TK}{\tau_0}\right)^{\mu L}+ \frac{nU\tau_0}{mS}.
\end{align*}
Obviously, we can select a proper event $\xi$ with a proper $\tau_0$ to minimize the upper bound. For $\tau\in\left[1, TK\right]$, by selecting $\tau_0=\left(\frac{2\sigma_l G}{nUL}\right)^\frac{1}{1+\mu L}\left(TK\right)^{\frac{\mu L}{1+\mu L}}$, we can minimize the bound with respect to $\tau_0$ as:
\begin{align*}
    \mathbb{E}\vert f(w^{T+1};z) - f(\widetilde{w}^{T+1};z)\vert
    &\leq \frac{2nU\tau_0}{mS} = \frac{2nU}{mS}\left(\frac{2\sigma_l G}{nUL}\right)^\frac{1}{1+\mu L}\left(TK\right)^{\frac{\mu L}{1+\mu L}}\\
    &\leq \frac{4}{S}\left(\frac{\sigma_l G}{L}\right)^\frac{1}{1+\mu L}\left(\frac{n^{\frac{\mu L}{1+\mu L}}}{m}\right)\left(UTK\right)^\frac{\mu L}{1+\mu L}.
\end{align*}

\subsection{Stability of Decentralized FL~(Theorem~\ref{thm2})}
\subsubsection{Aggregation Bound with Spectrum Gaps}
The same as the proofs in the last part, according to the Lemma~\ref{aggregation} and \ref{recursion}, we also can bound the local stability term. Let the learning rate $\eta=\frac{\mu}{\tau}=\frac{\mu}{tK+k}$ is decayed as the communication round $t$ and iteration $k$ where $\mu$ is a specific constant, we have:
\begin{equation}
\label{recurrence_full}
    \Delta_k^t + \frac{2\sigma_l}{SL}
    \leq \left(\frac{\tau}{\tau_0}\right)^{\mu L}\frac{2\sigma_l}{SL}.
\end{equation}
If we directly combine this inequality with Lemma~\ref{lemma:dfl_stability_bound}, we will get the vanilla stability of the vanilla SGD optimizer. However, this will be a larger upper bound which does not help us understand the advantages and disadvantages of decentralization. In the decentralized setups, an important study is learning how to evaluate the impact of the spectrum gaps. Thus we must search for a more precise upper bound than above. Therefore, we calculate a more refined upper bound for its aggregation step~(Lemma~\ref{aggregation}) with the spectrum gap $1-\lambda$.

Let $\mathbf{W}_k^t=\left[w_{0,k}^t, w_{1,k}^t, \cdots, w_{m,k}^t\right]^\top$ is the parameter matrix of all clients. In the stability analysis, we focus more on the parameter difference instead. Therefore, we denote the matrix of the parameter differences $\Phi_{k}^t=\mathbf{W}_{k}^{t}-\widetilde{\mathbf{W}}_{k}^{t}=\left[w_{0,k}^t - \widetilde{w}_{0,k}^t, w_{1,k}^t- \widetilde{w}_{1,k}^t, \cdots, w_{m,k}^t- \widetilde{w}_{m,k}^t\right]^\top$ as the difference between the models trained on $\mathcal{C}$ and $\widetilde{\mathcal{C}}$ on the $k$-th iteration of $t$-th communication round. Meanwhile, consider the update rules, we have:
\begin{align*}
    \Phi_{k+1}^t = \Phi_{k}^t - \eta_{k}^t\Gamma_k^t,
\end{align*}
where $\Gamma_k^t=\left[g_{0,k}^t-\widetilde{g}_{0,k}^t, g_{1,k}^t-\widetilde{g}_{1,k}^t, \cdots, g_{m,k}^t - \widetilde{g}_{m,k}^t\right]^\top$.

In the {\ttfamily D-FedAvg} method shown in Algorithm~\ref{algorithm:D-FedAvg}, the aggregation performs after $K$ local updates which demonstrates that the initial state of each round is $\mathbf{W}_0^{t}=\mathbf{A}\mathbf{W}_K^{t-1}$. It also works on their difference $\Phi_0^{t}=\mathbf{A}\Phi_K^{t-1}$. Therefore, we have:
\begin{align*}
    \Phi_{K}^t = \Phi_{0}^t - \sum_{k=0}^{K-1}\eta_{k}^t\Gamma_k^t = \mathbf{A}\Phi_{K}^{t-1} - \sum_{k=0}^{K-1}\eta_{k}^t\Gamma_k^t.
\end{align*}
Then we prove the recurrence between adjacent rounds. Let $\mathbf{P}=\frac{1}{m}\mathbf{1}\mathbf{1}^\top\in\mathbb{R}^{m\times m}$ and $\mathbf{I}\in\mathbb{R}^{m\times m}$ is the identity matrix, due to the double stochastic property of the adjacent matrix $\mathbf{A}$, we have:
\begin{align*}
    \mathbf{A}\mathbf{P} = \mathbf{P}\mathbf{A} = \mathbf{P}.
\end{align*}
Thus we have:
\begin{align*}
    \left(\mathbf{I}-\mathbf{P}\right)\Phi_{K}^t 
    &= \left(\mathbf{I}-\mathbf{P}\right)\mathbf{A}\Phi_{K}^{t-1} - \left(\mathbf{I}-\mathbf{P}\right)\sum_{k=0}^{K-1}\eta_{k}^t\Gamma_k^t\\
    &= \left(\mathbf{A}\Phi_{K}^{t-1} - \sum_{k=0}^{K-1}\eta_{k}^t\Gamma_k^t\right) - \mathbf{P}\mathbf{A}\Phi_{K}^{t-1} + \mathbf{P}\mathbf{A}\Phi_{K}^{t-1} - \mathbf{P}\left(\mathbf{A}\Phi_{K}^{t-1} - \sum_{k=0}^{K-1}\eta_{k}^t\Gamma_k^t\right).
\end{align*}
By taking the expectation of the norm on both sides, we have:
\begin{align*}
    \mathbb{E}\Vert \left(\mathbf{I}-\mathbf{P}\right)\Phi_{K}^t\Vert 
    &\leq \mathbb{E}\Vert \mathbf{A}\Phi_{K}^{t-1} - \sum_{k=0}^{K-1}\eta_{k}^t\Gamma_k^t - \mathbf{P}\mathbf{A}\Phi_{K}^{t-1}\Vert + \mathbb{E}\Vert\sum_{k=0}^{K-1}\eta_{k}^t\Gamma_k^t\Vert\\
    &\leq \mathbb{E}\Vert \mathbf{A}\Phi_{K}^{t-1} - \mathbf{P}\mathbf{A}\Phi_{K}^{t-1}\Vert + 2\mathbb{E}\Vert\sum_{k=0}^{K-1}\eta_{k}^t\Gamma_k^t\Vert\\
    &= \mathbb{E}\Vert \left(\mathbf{A}-\mathbf{P}\right)\left(\mathbf{I}-\mathbf{P}\right)\Phi_{K}^{t-1}\Vert + 2\mathbb{E}\Vert\sum_{k=0}^{K-1}\eta_{k}^t\Gamma_k^t\Vert\\
    &\leq \lambda\mathbb{E}\Vert \left(\mathbf{I}-\mathbf{P}\right)\Phi_{K}^{t-1}\Vert + 2\mathbb{E}\Vert\sum_{k=0}^{K-1}\eta_{k}^t\Gamma_k^t\Vert.
\end{align*}
The equality adopts $\left(\mathbf{A}-\mathbf{P}\right)\left(\mathbf{I}-\mathbf{P}\right) = \mathbf{A} - \mathbf{P} - \mathbf{A}\mathbf{P} + \mathbf{P}\mathbf{P} = \mathbf{A} - \mathbf{P}\mathbf{A}$. We know the fact that $\Phi_{k}^{t}=0$ where $\left(t, k\right) \in \left(t_0, k_0\right)$. Thus unwinding the above inequality we have:
\begin{align*}
    \mathbb{E}\Vert \left(\mathbf{I}-\mathbf{P}\right)\Phi_{K}^t\Vert
    &\leq \lambda^{t-t_0+1}\mathbb{E}\Vert \left(\mathbf{I}-\mathbf{P}\right)\Phi_{K}^{t_0-1}\Vert + 2\sum_{s=t_0}^{t}\lambda^{t-s}\mathbb{E}\Vert\sum_{k=0}^{K-1}\eta_k^s\Gamma_k^s\Vert\\
    &= 2\sum_{s=t_0}^{t}\lambda^{t-s}\mathbb{E}\Vert\sum_{k=0}^{K-1}\eta_k^s\Gamma_k^s\Vert.
\end{align*}
To maintain the term of $\mathbf{A}$, we have:
\begin{align*}
    \left(\mathbf{A}-\mathbf{P}\right)\Phi_{K}^t 
    &= \left(\mathbf{A}-\mathbf{P}\right)\mathbf{A}\Phi_{K}^{t-1} - \left(\mathbf{A}-\mathbf{P}\right)\sum_{k=0}^{K-1}\eta_{k}^t\Gamma_k^t\\
    &= \left(\mathbf{A}-\mathbf{P}\right)\left(\mathbf{A}-\mathbf{P}\right)\Phi_{K}^{t-1} - \left(\mathbf{A}-\mathbf{P}\right)\sum_{k=0}^{K-1}\eta_{k}^t\Gamma_k^t.
\end{align*}
The second equality adopts $\left(\mathbf{A}-\mathbf{P}\right)\left(\mathbf{A}-\mathbf{P}\right) = \left(\mathbf{A}-\mathbf{P}\right)\mathbf{A} - \mathbf{A}\mathbf{P} + \mathbf{P}\mathbf{P} = \left(\mathbf{A}-\mathbf{P}\right)\mathbf{A}$.
Therefore we have the following recursive formula:
\begin{align*}
    \mathbb{E}\Vert\left(\mathbf{A}-\mathbf{P}\right)\Phi_{K}^t\Vert
    &\leq \mathbb{E}\Vert\left(\mathbf{A}-\mathbf{P}\right)\left(\mathbf{A}-\mathbf{P}\right)\Phi_{K}^{t-1}\Vert + \mathbb{E}\Vert\left(\mathbf{A}-\mathbf{P}\right)\sum_{k=0}^{K-1}\eta_{k}^t\Gamma_k^t\Vert\\
    &\leq \lambda\mathbb{E}\Vert\left(\mathbf{A}-\mathbf{P}\right)\Phi_{K}^{t-1}\Vert + \lambda\mathbb{E}\Vert\sum_{k=0}^{K-1}\eta_{k}^t\Gamma_k^t\Vert.
\end{align*}
The same as above, we can unwind this recurrence formulation from $t$ to $t_0$ as:
\begin{align*}
    \mathbb{E}\Vert\left(\mathbf{A}-\mathbf{P}\right)\Phi_{K}^t\Vert
    &\leq \lambda^{t-t_0+1}\mathbb{E}\Vert\left(\mathbf{A}-\mathbf{P}\right)\Phi_{K}^{t_0-1}\Vert + \sum_{s=t_0}^t\lambda^{t-s+1}\mathbb{E}\Vert\sum_{k=0}^{K-1}\eta_{k}^s\Gamma_k^s\Vert\\
    &= \sum_{s=t_0}^t\lambda^{t-s+1}\mathbb{E}\Vert\sum_{k=0}^{K-1}\eta_{k}^s\Gamma_k^s\Vert.
\end{align*}

Both two terms required to know the upper bound of the accumulation of the gradient differences. In previous works, they often use the \textbf{Assumption of the bounded gradient} to upper bound this term as a constant. However, this assumption may not always hold as we introduced in the main text. Therefore, we provide a new upper bound instead. When $\left(t,k\right)<\left(t_0,k_0\right)$, the sampled data is always the same between the different datasets, which shows $\Gamma_k^t=0$. When $t=t_0$, only those updates at $k\geq k_0$ are different. When $t > t_0$, all the local gradients difference during local $K$ iterations are non-zero. Thus we can first explore the upper bound of the stages with full $K$ iterations when $t > t_0$. Let the data sample $z$ be the same random data sample and $z/\widetilde{z}$ be a different sample pair for abbreviation,  when $t \geq t_0$, we have:
\begin{align*}
    \mathbb{E}\Vert\eta\Gamma_k^t\Vert
    &= \mathbb{E}\Vert\eta\left[g_{0,k}^t-\widetilde{g}_{0,k}^t, g_{1,k}^t-\widetilde{g}_{1,k}^t, \cdots, g_{m,k}^t - \widetilde{g}_{m,k}^t\right]^\top\Vert\leq \eta\sum_{i\in\left[m\right]}\mathbb{E}\Vert g_{i,k}^t-\widetilde{g}_{i,k}^t\Vert\\
    &\leq \eta\sum_{i\neq i^\star}\mathbb{E}\Vert \nabla f_i(w_{i,k}^t, z)-\nabla f_i(\widetilde{w}_{i,k}^t, z)\Vert + \frac{(S-1)\eta}{S}\mathbb{E}\Vert \nabla f_{i^\star}(w_{i^\star,k}^t, z)-\nabla f_{i^\star}(\widetilde{w}_{i^\star,k}^t, z)\Vert\\
    &\quad + \frac{\eta}{S}\mathbb{E}\Vert \nabla f_{i^\star}(w_{i^\star,k}^t, z) - \nabla f_{i^\star}(\widetilde{w}_{i^\star,k}^t, z) +  \nabla f_{i^\star}(\widetilde{w}_{i^\star,k}^t, z) -  \nabla f_{i^\star}(\widetilde{w}_{i^\star,k}^t, \widetilde{z})\Vert\\
    &\leq \eta L\sum_{i\neq i^\star}\mathbb{E}\Vert w_{i,k}^t - \widetilde{w}_{i,k}^t\Vert + \frac{(S-1)\eta}{S}\mathbb{E}\Vert w_{i^\star,k}^t - \widetilde{w}_{i^\star,k}^t\Vert + \frac{\eta}{S}\mathbb{E}\Vert w_{i^\star,k}^t - \widetilde{w}_{i^\star,k}^t\Vert\\
    &\quad +  \frac{\eta}{S}\mathbb{E}\Vert \left(\nabla f_{i^\star}(\widetilde{w}_{i^\star,k}^t, z) - \nabla f_{i^\star}(\widetilde{w}_{i^\star,k}^t)\right) -  \left(\nabla f_{i^\star}(\widetilde{w}_{i^\star,k}^t, \widetilde{z}) - \nabla f_{i^\star}(\widetilde{w}_{i^\star,k}^t)\right)\Vert\\
    &\leq \eta L\sum_{i\in\left[m\right]}\mathbb{E}\Vert w_{i,k}^t - \widetilde{w}_{i,k}^t\Vert + \frac{2\eta\sigma_l}{S} = \eta L\left(\Delta_k^t + \frac{2\sigma_l}{SL}\right).
\end{align*}
According to the Lemma~\ref{aggregation}, \ref{recursion} and Eq.(\ref{recurrence_full}), we bound the gradient difference as:
\begin{align*}
    \mathbb{E}\Vert\eta\Gamma_k^t\Vert
    &\leq \eta L\left(\Delta_k^t + \frac{2\sigma_l}{SL}\right)\leq \left(\frac{\tau}{\tau_0}\right)^{\mu L}\frac{2\mu\sigma_l}{\tau S}.
\end{align*}
where $\tau=tK+k$. 

Unwinding the summation on $k$ and adopting Lemma~\ref{lambda}, we have:
\begin{align*}
    \sum_{s=t_0}^{t}\lambda^{t-s}\mathbb{E}\Vert\sum_{k=0}^{K-1}\eta_k^s\Gamma_k^s\Vert
    &\leq \sum_{s=t_0}^{t}\lambda^{t-s}\sum_{k=0}^{K-1}\mathbb{E}\Vert\eta_k^s\Gamma_k^s\Vert
    \leq \frac{2\mu\sigma_l}{S\tau_0^{\mu L}}\sum_{s=t_0}^{t}\lambda^{t-s}\sum_{k=0}^{K-1}\frac{\tau^{\mu L}}{\tau}\\
    &\leq \frac{2\mu\sigma_l}{S\tau_0^{\mu L}}\sum_{s=t_0}^{t}\lambda^{t-s}\sum_{k=0}^{K-1}\frac{\left(sK\right)^{\mu L}}{sK}=\frac{2\mu\sigma_l}{S}\left(\frac{K}{\tau_0}\right)^{\mu L}\sum_{s=t_0}^{t}\frac{\lambda^{t-s}}{s^{1-\mu L}}\\
    &\leq \frac{2\mu\sigma_l}{S}\left(\frac{K}{\tau_0}\right)^{\mu L}\sum_{s=t_0-1}^{t-1}\frac{\lambda^{t-s-1}}{\left(s+1\right)^{1-\mu L}}\leq \frac{2\mu\sigma_l\kappa_\lambda}{S}\left(\frac{K}{\tau_0}\right)^{\mu L}\frac{1}{t^{1-\mu L}}.
\end{align*}
Therefore, we get an upper bound on the aggregation gap which is related to the spectrum gap:
\begin{align}
    \label{aggregation_gap_lambda1}
    \mathbb{E}\Vert \left(\mathbf{I}-\mathbf{P}\right)\Phi_{K}^t\Vert
    &\leq 2\sum_{s=t_0}^{t}\lambda^{t-s}\mathbb{E}\Vert\sum_{k=0}^{K-1}\eta_k^s\Gamma_k^s\Vert\leq \frac{4\mu\sigma_l\kappa_\lambda}{S}\left(\frac{K}{\tau_0}\right)^{\mu L}\frac{1}{t^{1-\mu L}},\\
    \label{aggregation_gap_lambda2}
    \mathbb{E}\Vert\left(\mathbf{A}-\mathbf{P}\right)\Phi_{K}^t\Vert
    &\leq  \sum_{s=t_0}^t\lambda^{t-s+1}\mathbb{E}\Vert\sum_{k=0}^{K-1}\eta_{k}^s\Gamma_k^s\Vert\leq \frac{2\mu\sigma_l\lambda\kappa_\lambda}{S}\left(\frac{K}{\tau_0}\right)^{\mu L}\frac{1}{t^{1-\mu L}}. 
\end{align}
The first inequality provides the upper bound between the difference between the averaged state and the vanilla state, and the second inequality provides the upper bound between the aggregated state and the averaged state. 

\subsubsection{Stability Bound}
Now, rethinking the update rules in one round and we have:
\begin{align*}
    &\quad \ \sum_{i\in[m]}\mathbb{E}\Vert w_{i,K}^{t+1} - \widetilde{w}_{i,K}^{t+1}\Vert\\
    &= \sum_{i\in[m]}\mathbb{E}\Vert\left(w_{i,0}^{t+1} - \widetilde{w}_{i,0}^{t+1}\right) - \sum_{k=0}^{K-1}\eta_{k}^{t}\left(g_{i,k}^{t} - \widetilde{g}_{i,k}^t\right)\Vert\\
    &= \sum_{i\in[m]}\mathbb{E}\Vert\left(w_{i,0}^{t+1} - \widetilde{w}_{i,0}^{t+1}\right) - \left(w_{i,K}^{t} - \widetilde{w}_{i,K}^{t}\right) + \left(w_{i,K}^{t} - \widetilde{w}_{i,K}^{t}\right) - \sum_{k=0}^{K-1}\eta_{k}^{t}\left(g_{i,k}^{t} - \widetilde{g}_{i,k}^t\right)\Vert\\
    &\leq \sum_{i\in[m]}\left[\mathbb{E}\Vert\left(w_{i,0}^{t+1} - \widetilde{w}_{i,0}^{t+1}\right) - \left(w_{i,K}^{t} - \widetilde{w}_{i,K}^{t}\right)\Vert + \mathbb{E}\Vert\left(w_{i,K}^{t} - \widetilde{w}_{i,K}^{t}\right)\Vert + \mathbb{E}\Vert\sum_{k=0}^{K-1}\eta_{k}^{t}\left(g_{i,k}^{t} - \widetilde{g}_{i,k}^t\right)\Vert\right]\\
    &\leq \sum_{i\in[m]}\mathbb{E}\Vert\left(w_{i,K}^{t} - \widetilde{w}_{i,K}^{t}\right)\Vert + m\mathbb{E}\left[\frac{1}{m}\sum_{i\in[m]}\Vert\left(w_{i,0}^{t+1} - \widetilde{w}_{i,0}^{t+1}\right) - \left(w_{i,K}^{t} - \widetilde{w}_{i,K}^{t}\right)\Vert\right]\\
    &\quad + \sum_{i\in\left[m\right]}\mathbb{E}\Vert\sum_{k=0}^{K-1}\eta_{k}^{t}\left(g_{i,k}^{t} - \widetilde{g}_{i,k}^t\right)\Vert\\
    &\leq \sum_{i\in[m]}\mathbb{E}\Vert\left(w_{i,K}^{t} - \widetilde{w}_{i,K}^{t}\right)\Vert + m\mathbb{E}\sqrt{\frac{1}{m}\sum_{i\in[m]}\Vert\left(w_{i,0}^{t+1} - \widetilde{w}_{i,0}^{t+1}\right) - \left(w_{i,K}^{t} - \widetilde{w}_{i,K}^{t}\right)\Vert^2}\\
    &\quad + \sum_{i\in\left[m\right]}\mathbb{E}\Vert\sum_{k=0}^{K-1}\eta_{k}^{t}\left(g_{i,k}^{t} - \widetilde{g}_{i,k}^t\right)\Vert\\
    &= \sum_{i\in[m]}\mathbb{E}\Vert\left(w_{i,K}^{t} - \widetilde{w}_{i,K}^{t}\right)\Vert + \sqrt{m}\mathbb{E}\Vert\Phi_0^{t+1} - \Phi_K^{t}\Vert + \sum_{i\in\left[m\right]}\mathbb{E}\Vert\sum_{k=0}^{K-1}\eta_{k}^{t}\left(g_{i,k}^{t} - \widetilde{g}_{i,k}^t\right)\Vert\\
    &= \sum_{i\in[m]}\mathbb{E}\Vert\left(w_{i,K}^{t} - \widetilde{w}_{i,K}^{t}\right)\Vert + \sqrt{m}\mathbb{E}\Vert\mathbf{A}\Phi_K^{t} - \Phi_K^{t}\Vert + \sum_{i\in\left[m\right]}\mathbb{E}\Vert\sum_{k=0}^{K-1}\eta_{k}^{t}\left(g_{i,k}^{t} - \widetilde{g}_{i,k}^t\right)\Vert\\
    &\leq \sum_{i\in[m]}\mathbb{E}\Vert\left(w_{i,K}^{t} - \widetilde{w}_{i,K}^{t}\right)\Vert + \sqrt{m}\mathbb{E}\Vert\left(\mathbf{A} - \mathbf{P}\right)\Phi_K^{t}\Vert + \sqrt{m}\mathbb{E}\Vert\left(\mathbf{P} - \mathbf{I}\right)\Phi_K^{t}\Vert\\
    &\quad+ \sum_{i\in\left[m\right]}\mathbb{E}\Vert\sum_{k=0}^{K-1}\eta_{k}^{t}\left(g_{i,k}^{t} - \widetilde{g}_{i,k}^t\right)\Vert.
\end{align*}
Therefore, we can bound this by two terms in one complete communication round. One is the process of local $K$ SGD iterations, and the other is the aggregation step. For the local training process, we can continue to use Lemma~\ref{same_data_update}, \ref{different_data_update}, and \ref{recursion}. Let $\tau=tK+k$ as above, we have:
\begin{align*}
    &\quad \ \Delta_{K}^{t} + \frac{2\sigma_l}{SL}\\
    &\leq \left[\prod_{k=0}^{K-1}\left(1+\eta_{k}^t L\right)\right]\left(\Delta_{0}^{t} + \frac{2\sigma_l}{SL}\right)= \left[\prod_{k=0}^{K-1}\left(1+\frac{\mu L}{\tau}\right)\right]\left(\Delta_{0}^{t} + \frac{2\sigma_l}{SL}\right)\\
    &\leq \left[\prod_{k=0}^{K-1}e^{\frac{\mu L}{\tau}}\right]\left(\Delta_{0}^{t} + \frac{2\sigma_l}{SL}\right) = e^{\mu L\sum_{k=0}^{K-1}\frac{1}{\tau}}\left(\Delta_{0}^{t} + \frac{2\sigma_l}{SL}\right)\\
    &\leq e^{\mu L\ln\left(\frac{t+1}{t}\right)}\left(\Delta_{0}^{t} + \frac{2\sigma_l}{SL}\right) = \left(\frac{t}{t-1}\right)^{\mu L}\left(\Delta_{0}^{t} + \frac{2\sigma_l}{SL}\right)\\
    &\leq \left(\frac{t}{t-1}\right)^{\mu L}\left[\Delta_{K}^{t-1} + \sqrt{m}\left(\mathbb{E}\Vert\left(\mathbf{A} - \mathbf{P}\right)\Phi_K^{t}\Vert + \mathbb{E}\Vert\left(\mathbf{P} - \mathbf{I}\right)\Phi_K^{t}\Vert\right) + \frac{2\sigma_l}{SL}\right]\\
    &\leq \left(\frac{t}{t-1}\right)^{\mu L}\left(\Delta_{K}^{t-1} + \frac{2\sigma_l}{SL}\right) + \sqrt{m}\left(\frac{t}{t-1}\right)^{\mu L}\left(\mathbb{E}\Vert\left(\mathbf{A} - \mathbf{P}\right)\Phi_K^{t}\Vert + \mathbb{E}\Vert\left(\mathbf{P} - \mathbf{I}\right)\Phi_K^{t}\Vert\right)\\
    &\leq \underbrace{\left(\frac{t}{t-1}\right)^{\mu L}\left(\Delta_{K}^{t-1} + \frac{2\sigma_l}{SL}\right)}_{\text{local updates}} + \underbrace{\frac{6\sqrt{m}\mu\sigma_l\kappa_\lambda}{S}\left(\frac{K}{\tau_0}\right)^{\mu L}\left(\frac{t}{t-1}\right)^{\mu L}\frac{1}{t^{1-\mu L}}}_{\text{aggregation gaps}},
\end{align*}
The last adopts the Eq.(\ref{aggregation_gap_lambda1}) and (\ref{aggregation_gap_lambda2}), and the fact $\lambda \leq 1$.\\
Obviously, in the decentralized federated learning setup, the first term still comes from the updates of the local training. The second term comes from the aggregation gaps, which is related to the spectrum gap $\lambda$. Unwinding this from $t_0$ to $T$, we have:
\begin{align*}
    \Delta_{K}^{T} + \frac{2\sigma_l}{SL}
    &\leq \left(\frac{TK}{\tau_0}\right)^{\mu L}\frac{2\sigma_l}{SL} + \frac{6\sqrt{m}\mu\sigma_l\kappa_\lambda}{S}\left(\frac{K}{\tau_0}\right)^{\mu L}\sum_{t=t_0+1}^{T}\left(\frac{t}{t-1}\right)^{\mu L}\frac{1}{t^{1-\mu L}}\\
    &\leq \left(\frac{TK}{\tau_0}\right)^{\mu L}\frac{2\sigma_l}{SL} + \frac{12\sqrt{m}\mu\sigma_l\kappa_\lambda}{S}\left(\frac{K}{\tau_0}\right)^{\mu L}\sum_{t=t_0+1}^{T}\frac{1}{t^{1-\mu L}}\\
    &\leq \left(\frac{TK}{\tau_0}\right)^{\mu L}\frac{2\sigma_l}{SL} + \frac{12\sqrt{m}\mu\sigma_l\kappa_\lambda}{S}\left(\frac{K}{\tau_0}\right)^{\mu L}\frac{t^{\mu L}}{\mu L}\Bigg\vert_{t=t_0+1}^{t=T}\\
    &\leq \left(\frac{TK}{\tau_0}\right)^{\mu L}\frac{2\left(1+6\sqrt{m}\kappa_\lambda\right)\sigma_l}{SL}.
\end{align*}
The second inequality adopts the fact that $1<\frac{t}{t-1}\leq 2$ when $t > 1$ and the fact of $0 < \mu < \frac{1}{L}$.

According to the Lemma~\ref{lemma:dfl_stability_bound}, the first term in the stability (conditions is omitted for abbreviation) can be bounded as:
\begin{align*}
    &\quad \ \mathbb{E}\Vert w^{T+1} - \widetilde{w}^{T+1}\Vert
    \leq \frac{1}{m}\sum_{i\in\left[m\right]}\mathbb{E}\Vert\left(w_{i,K}^T - \widetilde{w}_{i,K}^T\right)\Vert \leq \left(\frac{TK}{\tau_0}\right)^{\mu L}\frac{2\left(1+6\sqrt{m}\kappa_\lambda\right)\sigma_l}{mSL}.
\end{align*}
Therefore, we can upper bound the stability in decentralized federated learning as:
\begin{align*}
    \mathbb{E}\left[\vert f(w^{T+1};z) - f(\widetilde{w}^{T+1};z)\vert\right] 
    &\leq G\mathbb{E}\left[\Vert w^{T+1} - \widetilde{w}^{T+1}\Vert \ \vert \ \xi \right] + \frac{U\tau_0}{S}\\
    &\leq \frac{2\sigma_l G}{SL}\left(\frac{1+6\sqrt{m}\kappa_\lambda}{m}\right)\left(\frac{TK}{\tau_0}\right)^{\mu L} + \frac{U\tau_0}{S}.
\end{align*}
The same as the centralized setup, to minimize the error of the stability, we can select a proper event $\xi$ with a proper $\tau_0$. For $\tau\in\left[1, TK\right]$, by selecting $\tau_0=\left(\frac{2\sigma_l G}{UL}\frac{1+6\sqrt{m}\kappa_\lambda}{m}\right)^\frac{1}{1+\mu L}\left(TK\right)^{\frac{\mu L}{1+\mu L}}$, we get the minimal:
\begin{align*}
    \mathbb{E}\left[\vert f(w^{T+1};z) - f(\widetilde{w}^{T+1};z)\vert\right] 
    &\leq \frac{2U\tau_0}{S}=\frac{2U}{S}\left(\frac{2\sigma_l G}{UL}\frac{1+6\sqrt{m}\kappa_\lambda}{m}\right)^\frac{1}{1+\mu L}\left(TK\right)^{\frac{\mu L}{1+\mu L}}\\
    &=\frac{4}{S}\left(\frac{\sigma_l G}{L}\right)^\frac{1}{1+\mu L}\left(\frac{1+6\sqrt{m}\kappa_\lambda}{m}\right)^{\frac{1}{1+\mu L}}\left(UTK\right)^{\frac{\mu L}{1+\mu L}}.
\end{align*}